\documentclass[10pt,journal,compsoc]{IEEEtran}

\ifCLASSOPTIONcompsoc
  \usepackage[nocompress]{cite}
\else
  \usepackage{cite}
\fi
\ifCLASSINFOpdf
\else
\fi
\usepackage{graphicx}
\usepackage{floatrow}
\usepackage{amsthm}
\usepackage{amsmath,amssymb}
\usepackage{mathtools}
\usepackage{subfig}
\usepackage{algorithm}
\usepackage{algpseudocode}
\usepackage{xcolor}
\usepackage{multirow}
\usepackage[tableposition=top]{caption}
\usepackage{floatrow}
\DeclareMathOperator*{\argmin}{\arg\min}
\newtheorem{theorem}{Theorem}
\algnewcommand{\LeftComment}[1]{\Statex \(\triangleright\) #1}

\definecolor{darkred}{RGB}{139, 0, 0}

\usepackage{pict2e}

\DeclareRobustCommand{\branch}{%
  \begingroup\setlength{\unitlength}{1ex}%
  \begin{picture}(1.2,1)
  \roundcap
  \polyline(1,2)(1,1)(1.9,1)(1,1)(1,0)
  \end{picture}%
  \endgroup
}

\DeclareRobustCommand{\leaf}{%
  \begingroup\setlength{\unitlength}{1ex}%
  \begin{picture}(1.2,1)
  \roundcap
  \polyline(1,2)(1,1)(1.9,1)
  \end{picture}%
  \endgroup
}

\newtheorem*{assumption*}{\assumptionnumber}
\providecommand{\assumptionnumber}{}


\makeatletter
\renewcommand\paragraph{\@startsection{paragraph}{4}{\z@}%
	{1.5ex plus .2ex minus .3ex}%
	{-0em}%
	{\normalsize\bf}}
\makeatother

\begin{document}
%
\title{CacheNet: A Model Caching Framework for Deep Learning Inference on the
Edge}
%
%
%
%

\author{Yihao~Fang,~\IEEEmembership{Student Member,~IEEE,}
        Shervin~Manzuri~Shalmani,~\IEEEmembership{Student Member,~IEEE,}
        and~Rong~Zheng,~\IEEEmembership{Senior~Member,~IEEE}
\IEEEcompsocitemizethanks{\IEEEcompsocthanksitem Y. Fang, S. Manzuri Shalmani, and R. Zheng are with the Department of Computing and Software, McMaster University, Hamilton,
ON, Canada.\protect\\
E-mail: \{fangy5,manzuris,rzheng\}@mcmaster.ca. \protect\\
R. Zheng is a visiting professor in Harbin Institute of Technology (Shenzhen), China between 2019 and 2020. 
}
}

\IEEEtitleabstractindextext{%
\begin{abstract}
The success of deep neural networks (DNN) in machine perception applications such as image classification and speech recognition comes at the cost of high computation and storage complexity. Inference of uncompressed large scale DNN models can only run in the cloud with extra communication latency back and forth between cloud and end devices, while compressed DNN models achieve real-time inference on end devices at the price of lower predictive accuracy. In order to have the best of both worlds (latency and accuracy), we propose CacheNet, a model caching framework. CacheNet caches low-complexity models on end devices and high-complexity (or full) models on edge or cloud servers. By exploiting temporal locality in streaming data, high cache hit and consequently shorter latency can be achieved with no or only marginal decrease in prediction accuracy. Experiments on CIFAR-10 and FVG have shown CacheNet is $58-217\%$ faster than baseline approaches that run inference tasks on end devices or edge servers alone.  
\end{abstract}

}

\maketitle

\IEEEdisplaynontitleabstractindextext

%
\IEEEpeerreviewmaketitle

\ifCLASSOPTIONcompsoc
\IEEEraisesectionheading{\section{Introduction}\label{sect:cn_introduction}}
\else
\section{Introduction}
\label{sect:cn_introduction}
\fi

\IEEEPARstart{I}{n} recent years, deep neural networks (DNN) have achieved tremendous successes
in  perception applications such as image classification, speech
recognition, target tracking  and machine translation. In many cases, they
outperform human beings in accuracy. However, such high accuracy comes at the
cost of high computation and storage complexity due to large model sizes. For
instance, ResNet-152 contains $152$ layers and over $60$M parameters. Inference
using such large-scale DNN models cannot be accomplished on end devices with
limited computation power and storage in real-time. As a result, many model
compression techniques have been proposed to reduce the size of DNN networks
often at the expense of prediction performance \cite{lin2016fixed, sainath2013low}. Therefore, application
developers face a dilemma to choose between a highly accurate model that can only
run in the cloud with extra communication latency of uploading raw input data
and getting the results back, or local execution of compressed models with
reduced accuracy. 

{\it Is it possible to get the best of both worlds?} In other words, can we
achieve a good trade-off between latency and prediction accuracy? This question
has to some degree been answered by partitioning
approaches~\cite{kang2017neurosurgeon, teerapittayanon2017distributed, fang2019teamnet}. They mainly fall into two paradigms: 1) {\it model partitioning}: concurrent computing among edge nodes and/or end devices~\cite{fang2019teamnet}, which collaboratively performs inference in parallel per a particular sensor input; 2) {\it computation partitioning}: partition between edge and cloud, which take a pre-trained deep model and decide at run-time
based on computation capability of local and cloud compute nodes and
communication overheads where portions of computation should reside~\cite{kang2017neurosurgeon}. The inference time of both paradigms is clearly lower bounded by the smaller (or smallest) of inference times on the
end device and a cloud node (or on all end devices/edge nodes). Furthermore, as per computation partitioning, since DNN models tend to be sequential, the possible ways of partitioning are limited.

In this work, we take a drastically new approach in addressing the trade-off
between latency and prediction accuracy of DNNs. Our approach is motivated by
two observations of perception applications with inputs from
natural scenes or human interactions. First, despite the fact that such applications may need
to handle a large number of input classes over time, the classes of inputs commonly encountered can be much smaller. For instance, an average
English speaking person uses about 4000 words in daily life out of 171,476
words listed in the second edition of Oxford English Dictionary. Secondly,
there exists strong temporal locality in terms of the types of inputs
encountered in a short period of time. This is especially true for vision
processing where rich redundancy exists among consecutive video
frames~\cite{xu2018deepcache, chen2015glimpse, huynh2017deepmon,
mathur2017deepeye}. 

To exploit these two properties, we propose {\it CacheNet}, a model caching
framework for deep learning inference on edge.
CacheNet is inspired by caching in the memory hierarchy. In computer architecture, the memory hierarchy separates computer storage (e.g., register, cache, random access memory, etc.) based on response time \cite{mutlu2015research}. Caching increases data
retrieval performance (e.g. faster response time) by reusing previously retrieved and computed data in
the storage. Analogous to the memory hierarchy, end devices
are closer to data sources and thus have faster response time but lower storage
capacity; while an edge server has more storage capacity but relatively longer network latency.
However, unlike the memory hierarchy that only stores data, CacheNet stores DNN
models. To mitigate the limited computation power on end devices, only
down-sized models with high confidence in the current input data are
stored. Thanks to the temporal locality and the small number of frequently
observed classes, the cached model only needs to be replaced infrequently.

In short, CacheNet combines model partitioning with caching. Instead of
training a single large-scale model, CacheNet generates multiple small
submodels each capturing a partition of the knowledge represented by the large
model. In the proposed architecture, the end device is responsible for selecting a locally cached model and performing the inference; whereas the edge server stores the baseline model and submodels, and is responsible to handle ``cache misses'' when there are sufficient changes in input data. CacheNet is agnostic to the architecture of a baseline deep model. Both the number of submodels and the baseline deep model can be specified by users.

We have implemented CacheNet in TensorFlow, TensorFlow Lite and NCNN. Here, TensorFlow is a high-performance framework for neural network training, while TensorFlow Lite and NCNN are lightweight inference framework optimized for edge computing. CacheNet has been evaluated on a variety of end devices and two different datasets (CIFAR-10 \cite{krizhevsky2009learning} and FVG \cite{gait-recognition-via-disentangled-representation-learning}). We found that CacheNet outperforms end-device-only and edge-server-only approaches in inference time without compromising inference accuracy. For CIFAR-10, CacheNet is $2.2$ times faster than the end-device-only approach and $58\%$ faster than edge-server-only; for FVG, it is $1.5$ times faster than end-device-only and $71\%$ faster than edge-server-only.

The rest of the paper is organized as follows. Section~\ref{sect:cn_related_works}
describes related works to CacheNet from two perspectives: caching and
partitioning. An overview of our approach is given in Section~\ref{sect:cn_system_design} from requirements to system level design. In Section~\ref{sect:cn_training}, we elaborate on aspects of training CacheNet and formalize CacheNet mathematically. Details of inference is
provided in Section~\ref{sect:cn_inferencing} from {\it partition
selection} on the edge server to {\it cache replacement} on end devices. Section~\ref{sect:cn_evaluation} provides evaluations of CacheNet on multiple end devices including Jetson TX2, Jetson Nano, and Raspberry Pi 4. The conclusion and future works are stated in Section~\ref{sect:cn_conclusion}.

\section{Related Works}
\label{sect:cn_related_works}

Existing algorithmic approaches to accelerate machine learning inference on end devices mainly fall into three categories, namely: i) model compression, ii) computation and model partitioning, and iii) reduction of computation in machine learning pipelines.  The three categories of approaches are orthogonal to one another and can be applied jointly.  
Among the three, the latter two are closer to CacheNet and will be discussed in further details in this section.
\subsection{Computation and Model Partitioning}
Partitioning splits a known neural network model into multiple parts to be executed either sequentially or concurrently on the edge and cloud. It can be performed between layers. By trading off between the time offloading computation to the cloud with the time spent in local computation on edge, a shorter latency could be achieved \cite{kang2017neurosurgeon}.  

A more sophisticated computation partitioning was proposed in distributed deep neural networks (DDNNs) \cite{teerapittayanon2017distributed}.
DDNN was designed to perform fast and localized inference using shallow portions of a neural network on end devices. Using an exit point after device inference, an output is classified locally. If the classification cannot be made due to low confidence, the task is escalated to a higher exit point (e.g. the edge exit) in the hierarchy until the last exit (the cloud exit). With multiple exit points, DDNNs can significantly reduce communication costs.

TeamNet~\cite{fang2019teamnet} takes a different approach for computation partitioning. Rather than dividing a pre-trained neural network structurally, it explores knowledge specialization and trains multiple small NNs through competitive and selective learning. During inference, the NNs are executed in parallel on cooperative end devices. 
By decision-level fusion, a master node (either one of the end devices or a edge/cloud node) outputs the final inference results. Since computation partitioning in TeamNet is done at model level, it is also considered a model partitioning approach. 

CacheNet bears similarity with TeamNet in training multiple shallower models to represent the knowledge of a single deep model. However, unlike TeamNet that requires concurrent execution of the shallower models, CacheNet utilizes a ``selector'' to determine the suitable shallow model based on input data. In CacheNet, when a cache hit occurs, the inference is performed on the end device only. The overall inference time is reduced by the indexability of specialized submodels and running the suitable submodel locally most of the time.  

\begin{figure*}[!t]
\centering

\subfloat[Train]{
\includegraphics[width=0.81\linewidth]{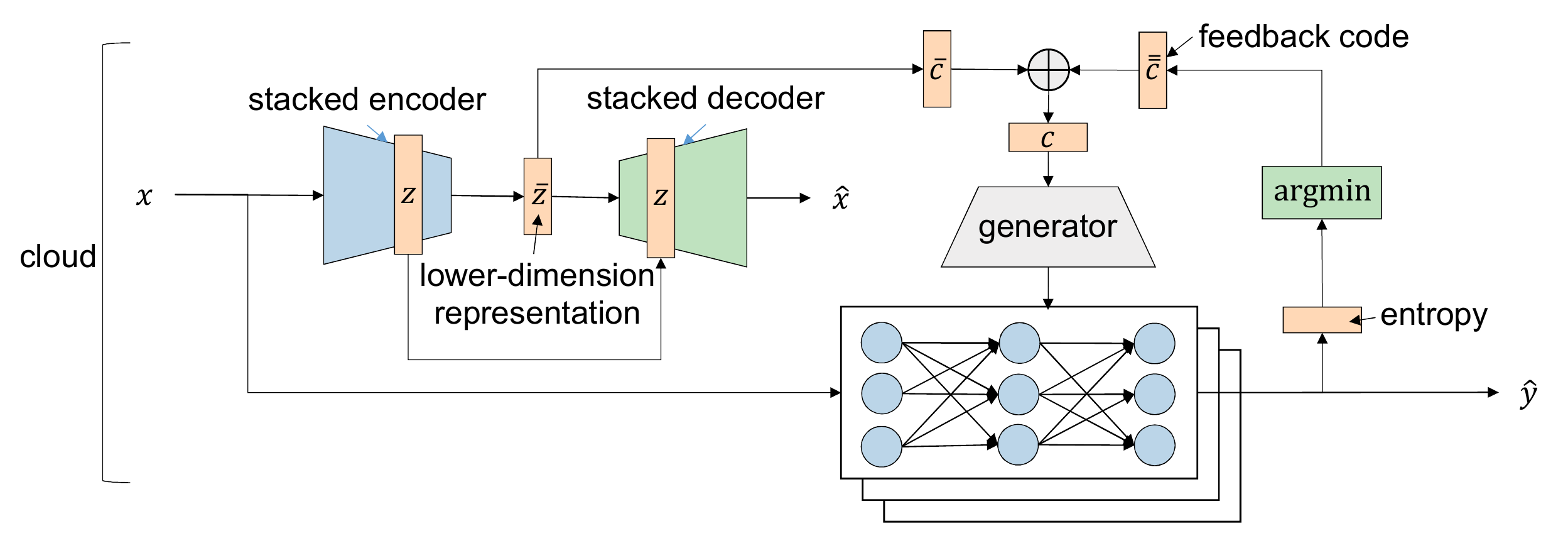}
\label{fig:cn_system_train}
}
\\
\subfloat[Inference]{
\includegraphics[width=0.81\linewidth]{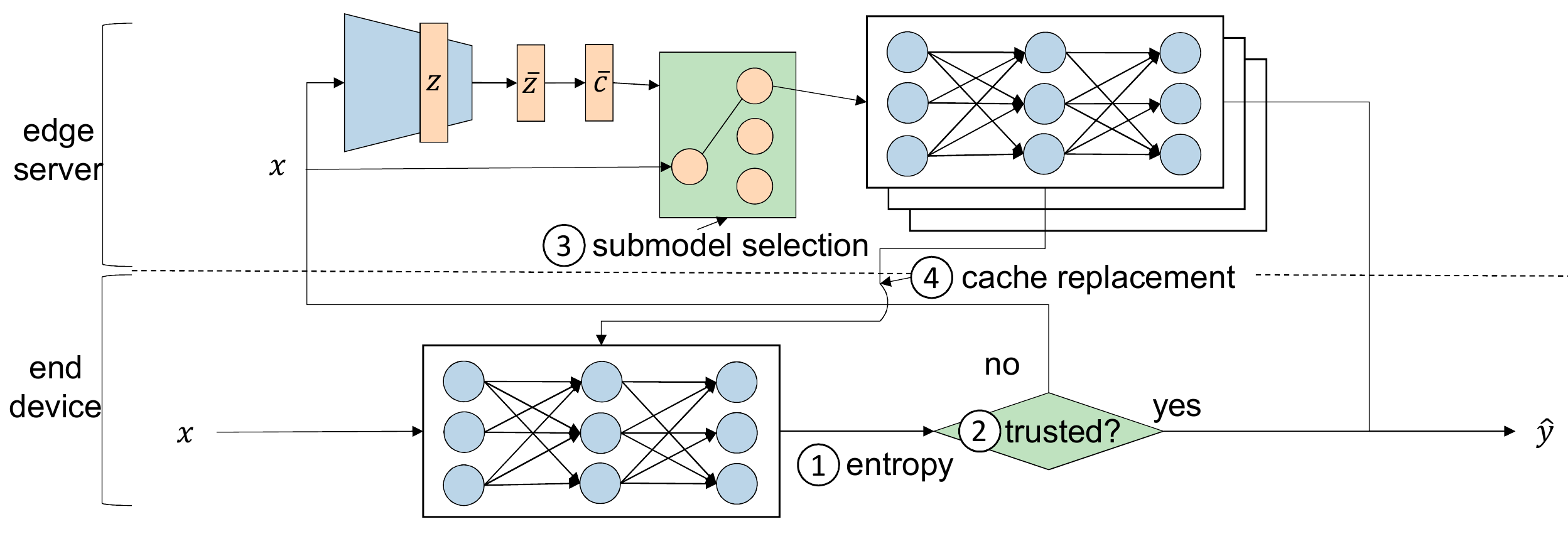}
\label{fig:cn_system_inference}
}
\caption[CacheNet's System Architecture Diagram]{(a) CacheNet first partitions a neural network into multiple smaller specialized neural networks in the cloud. (b) Owing to the temporal locality that exists in the video, the smaller specialized neural network will work well on consecutive frames over a short period. An abrupt change of frame induces higher entropy and triggers cache replacement. }
\label{fig:cn_system}
\end{figure*}

\subsection{Computation Reduction}
Exploiting the existence of the temporal locality in input data, several works reduce DNN inference time by reusing all or part of previous computation results.

Glimpse is a continuous, real-time object recognition system for camera-equipped mobile devices~\cite{chen2015glimpse}. 
In Glimpse, object recognition tasks are executed on local devices when the communication latency between the server and mobile device is higher than a frame-time. In addition to using a reduced model for faster local inference, Glimpse uses an active cache of video frames on the mobile device. A subset of the frames in the active cache is used to track objects on the mobile, using (stale) hints about objects that arrive from the server from time to time.
In \cite{xu2018deepcache}, Xu {\it et al.} proposed DeepCache, a principled cache design for deep learning inference in continuous mobile vision. It breaks down an input video frame into smaller blocks and discovers similar blocks between consecutive frames using diamond search~\cite{xu2018deepcache}. Computation on reusable regions (e.g., feature maps) can thus be cached and propagated through subsequent layers without further processing. 
In \cite{apicharttrisorn2019frugal}, 
to reduce energy drain while maintaining good object tracking precision, the authors develop a software framework called MARLIN. MARLIN only uses a DNN as needed, to detect new objects or recapture objects that significantly change in appearance. It employs lightweight methods in between DNN executions to track the detected objects with high fidelity. Alternatively, we can view MARLIN as reuse the detection and classification results by associating detected objects across multiple frames. In \cite{guo2018foggycache}, Guo et al. proposed FoggyCache for cross-device approximate computation reuse. FoggyCache reuses previously computed outputs by harnessing the ``equivalence'' between different input values. Content lookup and high quality reuse are achieved by the adoption of  adaptive locality sensitive hashing (A-LSH) and homogenized k-nearest neighbors (H-kNN). Harnessing reuse opportunities translates to reduced computation latency and energy consumption. 

All afore-mentioned approaches are orthogonal to CacheNet. In DeepCache and MARLIN, a full-fledged deep model is still needed on an end device and thus the worst-case execution time is not reduced. This is in contrast with CacheNet, which only runs reduced submodels locally.

\section{System Design}
\label{sect:cn_system_design}

CacheNet is a distributed inference framework on edge. Its training phase happens in the cloud and the inference is a collaboration between the edge server and the end device. The intuition behind CacheNet is dividing a neural network's knowledge into multiple specialized partitions (neural networks). These specialized partitions are generally a few times smaller than the original neural network, and only the specialized partition is transferred to the end device for inference. From the end device's perspective, it caches only a times smaller and specialized partition of the knowledge, and thus its inference is times faster than the original ones.

The challenges of partitioning are two folds: 1) each partition must be sufficiently specialized and the combination (collaboration) of all partitions must behave roughly equivalently to the original neural network; 2) There must be a selector that picks the right partition given a specific hint at a time. The first challenge was mostly solved by TeamNet \cite{fang2019teamnet}, while the second one has not been solved by any approaches at this point. 

In order to solve the second challenge, it is necessary to formalize the hint as a specific representation. Inspired by coding theory, a code vector is a good representation as long as the mutual information between the code vector and the input image is maximized at the training phase. Although we have the hint representation, it is still difficult to associate the representation with a specific portion of the knowledge. To do so, we introduce a generator that generates the neural network's parameters accordingly to the given code representation. 

Thus, in training (Figure~\ref{fig:cn_system_train}), we need to 1) maximize the mutual information between the code representation and the input image; 2) better associate the code representation with the specialized neural network partition; 3) train each partition with respect to their output entropy, which has been demonstrated practical in TeamNet \cite{fang2019teamnet}. CacheNet's system design is therefore conducted simultaneously with respect to the above objectives. 

During inference (Figure~\ref{fig:cn_system_inference}), CacheNet should infer the code representation from a particular input image, and then the code representation will be used as a hint to tell which specialized partition to be cached on the end device. Without the need to transfer the input frames to the edge server every time, inference latency can be shortened. The accuracy is generally not sacrificed, because there exists a temporal locality on consecutive frames most of the time. As long as there is not an abrupt change of the scene, a specialized partition should work well; otherwise (e.g., in regard to edited clips from multiple cameras, or a fast-moving object/camera \cite{kwon2008tracking}), a cache replacement should be triggered, considering a partition only holds a subset of the knowledge. 

\section{Training CacheNet}
\label{sect:cn_training} 

As illustrated in Figure~\ref{fig:cn_system_train}, to train CacheNet
submodels, we need to first divide the input data into partitions\footnote{The
partitions are overlapping as will be discussed in Section~\ref{subsect:cn_indexability_of_low_dimension_representation}.}. The index
associated with a partition is taken as an input to a neural network generator
to produce the corresponding submodel for the partition. The encoder that maps
input data to partition indices and the submodels will be optimized jointly.
Next, we discuss the steps in detail. 
\subsection{Stacked Information Maximizing Variational Autoencoder (S-InfoVAE)}
The purpose of this step is to map input data into a low dimension space for
further partitioning. The low-dimension representation should
preserve the proximity among data and allow ``reconstruction'' of the orignal
data.  

Variational Bayesian autoencoder was proposed by Kingma and
Welling~\cite{kingma2013auto}. The basic idea is to find a lower-dimension
latent variable underlying the corresponding input distribution. Let $z$
denote the latent variable and $x$ represent the input variable.  Consider a dataset $D = \{X, Y\}$, where X is drawn independently from an input probability distribution $p_D (x)$.  Suppose that $p_\xi(z)$ (the prior distribution of $z$) and the
conditional probability distribution $p_\xi (x|z)$ are both parameterized by a
neural network with parameters $\xi$. One can find the
optimal parameters $\xi$ by maximizing the log-likelihood as:  

\begin{equation}
\mathbb{E}_{p_D(x)}\left[\log p_\xi(x) \right] = \mathbb{E}_{p_D(x)} \left[\log \mathbb{E}_{p_\xi (z)} \left[p_\xi(x|z) \right] \right].
\end{equation}

However, the integral of the marginal likelihood $p_\xi (x)$ is generally
intractable even for a moderately complex neural network with a single non-linear
hidden layer. A possible approach \cite{kingma2013auto} is to rewrite $\log
p_\xi(x)$: 
\begin{equation}
\log p_\xi (x) = D_{KL} (q_\psi (z|x)||p_\xi (z|x)) + \mathcal{L}(\xi, \psi; x),
\end{equation}
where 
\begin{equation}
\mathcal{L}(\xi, \psi; x) = - D_{KL} (q_\psi (z|x)||p_\xi (z)) + E_{q_\psi (z|x)} \log p_\xi (x|z).
\end{equation}

Since Kullback-Leibler divergence is always non-negative,
$\mathcal{L}(\xi, \psi; x)$ is a lower bound of $\log p_\xi (x)$, namely, 

\begin{equation}
\mathcal{L}(\xi, \psi; x) \leq \log p_\xi (x).
\end{equation}

By maximizing the lower bound $\mathcal{L}(\xi, \psi; x)$, the log likelihood
$\log p_\xi (x)$ is maximized as well. However, since the latent variable $z$
is of lower dimension than the input variable $x$, any optimization against $x$
may be magnified compared to $z$. To counteract the imbalance problem, Zhao et al.~\cite{zhao2019infovae} propose to put more weight on $z$. Let
$\mathcal{L}(\xi, \psi) $ be the expectation of $\mathcal{L}(\xi, \psi; x)$
with respect to the input distribution $p_D (x)$. We then have, 

\begin{equation}
\begin{split}
\mathcal{L}(\xi, \psi)  = & E_{p_D (x)} \mathcal{L}(\xi, \psi; x) \\
 = & - D_{KL} (q_\psi (x, z) || p_\xi (x, z)) \\
 = & - D_{KL} \left( q_\psi (z) || p_\xi (z) \right) \\
   & \quad - \mathbb{E}_{p_\xi (z)} \left[ D_{KL} \left( q_\psi (x|z) || p_\xi (x|z) \right) \right].
\end{split}
\end{equation}

To put more weights on $z$, one needs to add i) a scaling parameter to the
Kullback-Leibler divergence between $q_\psi (z)$ and $p_\xi (z)$, and ii) a
term of mutual information between $x$ and $z$ \cite{zhao2019infovae}:

\begin{equation}
\begin{split}
\mathcal{L}^*(\xi, \psi) = & - \lambda D_{KL} \left( q_\psi (z) || p_\xi (z) \right) \\
& \quad - \mathbb{E}_{p_\xi (z)} \left[ D_{KL} \left( q_\psi (x|z) || p_\xi (x|z) \right) \right] \\
& \quad + \alpha I_{q_\psi (x, z)} (x;z).
\end{split}
\end{equation}

In practice, $\mathcal{L}^*(\xi, \psi)$ can be rewritten into \eqref{eq:L_star} for more effective optimization \cite{zhao2019infovae}: 

\begin{equation} \label{eq:L_star}
\begin{split}
\mathcal{L}^*(\xi, \psi) = & \mathbb{E}_{p_D (x)} \mathbb{E}_{q_\psi (z|x)} \left[ \log p_\xi (x|z) \right] \\ 
& \quad - (1 - \alpha) \mathbb{E}_{p_D (x)} D_{KL} (q_\psi (z|x)||p_\xi (z)) \\
& \quad - (\alpha + \lambda - 1) D_{MMD} (q_\psi (z) || p_\xi (z)),
\end{split}    
\end{equation}
where $D_{MMD} (q_\psi (z) || p_\xi (z)) $ is the maximum-mean discrepancy between $q_\psi (z)$ and $p_\xi (z)$. 

Experiments show that when the latent variable $z$ is of far lower dimension
than the input variable $x$, the lower bound $\mathcal{L}^*(\xi, \psi)$ can not properly converge. To deal with this problem, we propose the S-InfoVAE by keeping $z$ at a relative high dimension and introducing a second latent
variable $\bar{z}$ of dimension two. The corresponding parameters (or
equivalently the neural networks) of the two latency variables are stage-wisely optimized. Formally, the second optimization objective is defined as follows:

\begin{equation}
\mathcal{\bar{L}}^*(\bar{\xi}, \bar{\psi}) = E_{p_{\bar{\psi}} (z)} \mathcal{L}(\bar{\xi}, \bar{\psi}; z)
\end{equation}

\subsection{Indexability of Low-dimension Representation}
\label{subsect:cn_indexability_of_low_dimension_representation}

To divide data into overlapping partitions, sophisticated indexes are needed. 
Let $K$ be the total number of submodels, an input parameter of CacheNet. Each
input sample in $D$ is associated with one or more indices chosen from $1$ to $K$
and will be used to train the corresponding submodel(s). By allowing
multiple indices per data sample or equivalently shared training data, we
facilitate knowledge sharing across submodels.  In this step, we determine
the indices of input data solely based on the low-dimension representations from the S-InfoVAE. In subsequent sections, we will also incorporate feedback from the
resulting submodels in the form of uncertainty.

Recall that $\bar{z}$'s are 2D vectors. To calculate the angular distance
between the vector $\bar{z} = [\bar{z}_1 \bar{z}_2]$ and the x-axis, the
$\arctan$ trigonometric function is applied:

\begin{equation}
\theta =
\begin{dcases} 
      \arctan \frac{\bar{z}_2}{\bar{z}_1} & \bar{z}_1 > 0 \\
      \arctan \frac{\bar{z}_2}{\bar{z}_1} + \pi & \bar{z}_1 < 0, \bar{z}_2 \geq 0 \\
      \arctan \frac{\bar{z}_2}{\bar{z}_1} - \pi & \bar{z}_1 < 0, \bar{z}_2 < 0 \\
      \frac{\pi}{2}  & \bar{z}_1 = 0, \bar{z}_2 > 0 \\
      - \frac{\pi}{2}  & \bar{z}_1 = 0, \bar{z}_2 < 0 \\
      0  & \bar{z}_1 = 0, \bar{z}_2 = 0.
   \end{dcases}
\end{equation}

For better convergence, a small noise term $\epsilon$ is added to the $\theta$.
To keep the resulting angles between $0$ and $2\pi$, a modulo
function is applied as follows:

\begin{equation}
    \tilde{\theta} = (\theta + \epsilon) \mod 2 \pi.
\end{equation}

For $K$ partitions where each partition roughly occupies a region of $\frac{2
\pi}{K}$, the midpoint of the $k^{th}$ partition is given by $\frac{2 \pi
\left( k-\frac{1}{2} \right)}{K} $, for $k = 1, \ldots, K$. Let $\zeta$ be a
vector of all such midpoints, namely:

\begin{equation} \label{eq:zeta}
\mathbb{\zeta} = [\zeta_1 \ldots \zeta_K] \text{,} \quad
\zeta_k = \frac{2 \pi \left( k-\frac{1}{2} \right)}{K}.
\end{equation}

We wish to assign input samples to partitions based on their closeness to the
$K$ midpoints in polar coordinates. One straightforward approach is via a
1-nearest neighbor search, namely, finding $k$ that minimizes $\min \left(|\tilde{\theta} - \zeta_k|, 2\pi - |\tilde{\theta} - \zeta_k|\right)$. Doing so will result in a one-hot vector with one for the $k$th
  element and zeros for all other elements. Instead, we choose to define a {\it soft} code $\bar{c}$ as, 
\begin{equation} \label{eq:c_bar}
\bar{c} = \sum_{n=-1}^{n=1} \exp\left(-\frac{(\zeta-\tilde{\theta}+2\pi n)^2}{2 \sigma^2}\right),
\end{equation}
where $\sigma$ is a parameter that controls the speed of decay as $\tilde{\theta}$ deviates from the midpoints. 
Clearly, each element of $\bar{c}$ is between 0 and 1, and the maximum value occurs at $k = \argmin_k \left( \min \left(|\tilde{\theta} - \zeta_k|, 2\pi - |\tilde{\theta} - \zeta_k| \right) \right)$.

With the soft code $\bar{c}$ of some input $x$ and a threshold $\tau$, we can determine which partition(s) it belongs to as $\{k|\bar{c}_k \ge \tau\}$. Plugging \eqref{eq:zeta} and \eqref{eq:c_bar}, we have $c_k \ge \tau$ if
the following condition holds,
\newcommand{\half}{\frac{1}{2}}
$$
\frac{2\pi\left(k-\half\right)}{K} - \sigma\sqrt{-2\log\tau} \le \tilde{\theta} \le
\frac{2\pi\left(k-\half\right)}{K} + \sigma\sqrt{-2\log\tau}.
$$

In other words, we can view mapping to soft codes along with a suitable
choice of $\tau$ and $\sigma$, having the effect of dividing the polar coordinate space into $K$ overlapping sectors with width $2\sigma\sqrt{-2\log\tau}$. An example of four partitions is given in Figure~\ref{fig:cn_angle2code}.  When $\bar{z}$ of an input $x$ falls into the overlapping area of sectors $i$ and $j$, we view it as contributing to the training of submodel $i$ and $j$. Let $\gamma$ be the overlapping ratio (normalized by $2\pi$). $\sigma$ can thus be determined by,
\begin{equation} \label{eq:sigma}
    \sigma = \sqrt{- \frac{ \pi^2 (1 + \gamma)^2}{2 K^2 \log \tau}}
\end{equation}

In Figure~\ref{fig:cn_angle2code}, $\gamma$ is set to $30\%$ and $\tau$ equals to $0.3$. When $\tilde{\theta}$ equals to $\frac{1}{3} \pi$, which is outside of
the overlapping region (Figure~\ref{fig:cn_angle2code_1}), the data point only
contributes to the training of one submodel. When $\tilde{\theta}$ equals to
$\frac{4}{9} \pi$, which is in between two midpoints $\frac{1}{4} \pi$ and
$\frac{3}{4} \pi$, the data point contributes to the training of the two
corresponding submodels. 

\begin{figure}[ht] 
\centering
\subfloat[Vector not in any overlap region ($\tau$ equals to $0.3$)]{
\includegraphics[width=0.45\linewidth]{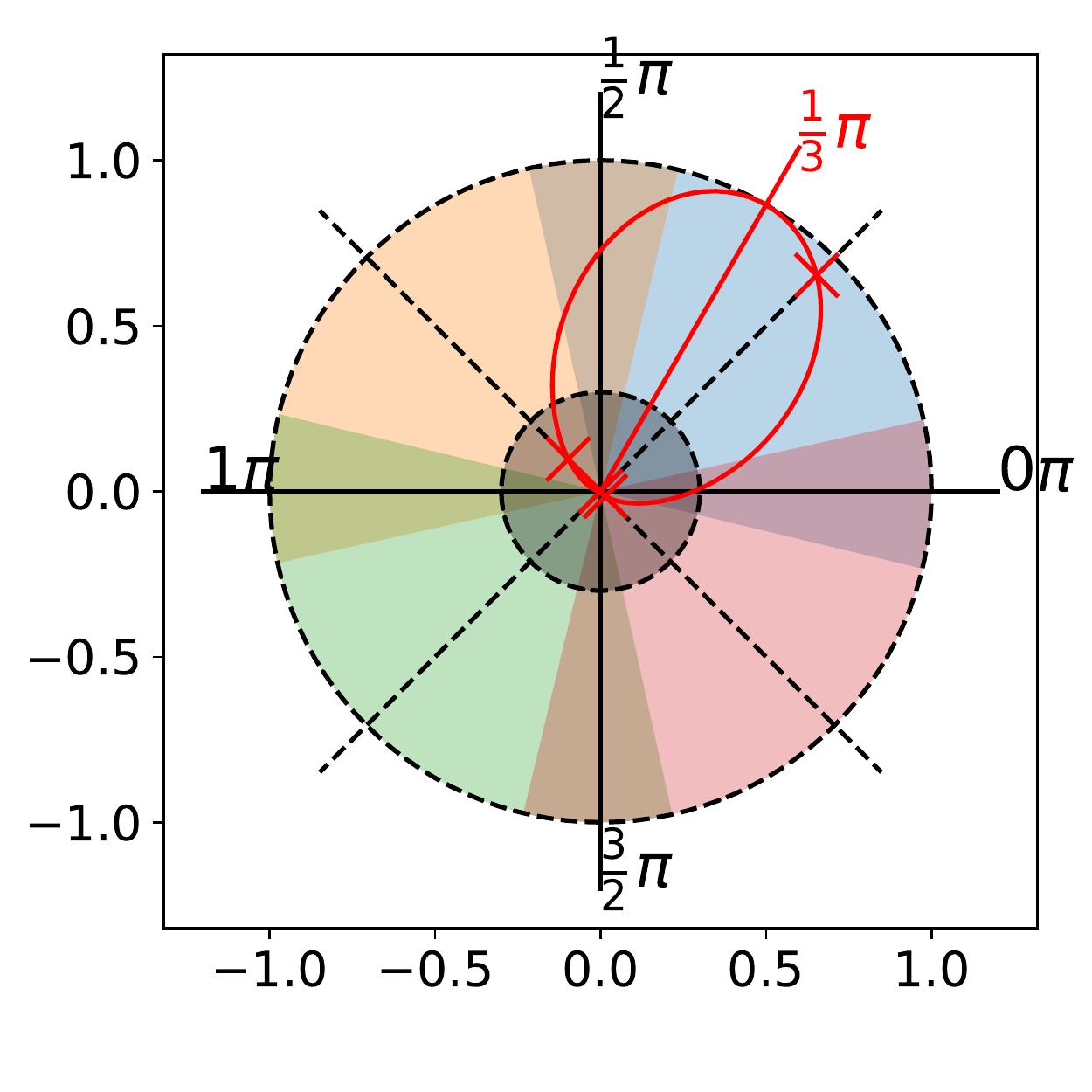}
\label{fig:cn_angle2code_1}
}\hspace{0.02\linewidth}
\subfloat[Vector in an overlap region ($\tau$ equals to $0.3$)]{
\includegraphics[width=0.45\linewidth]{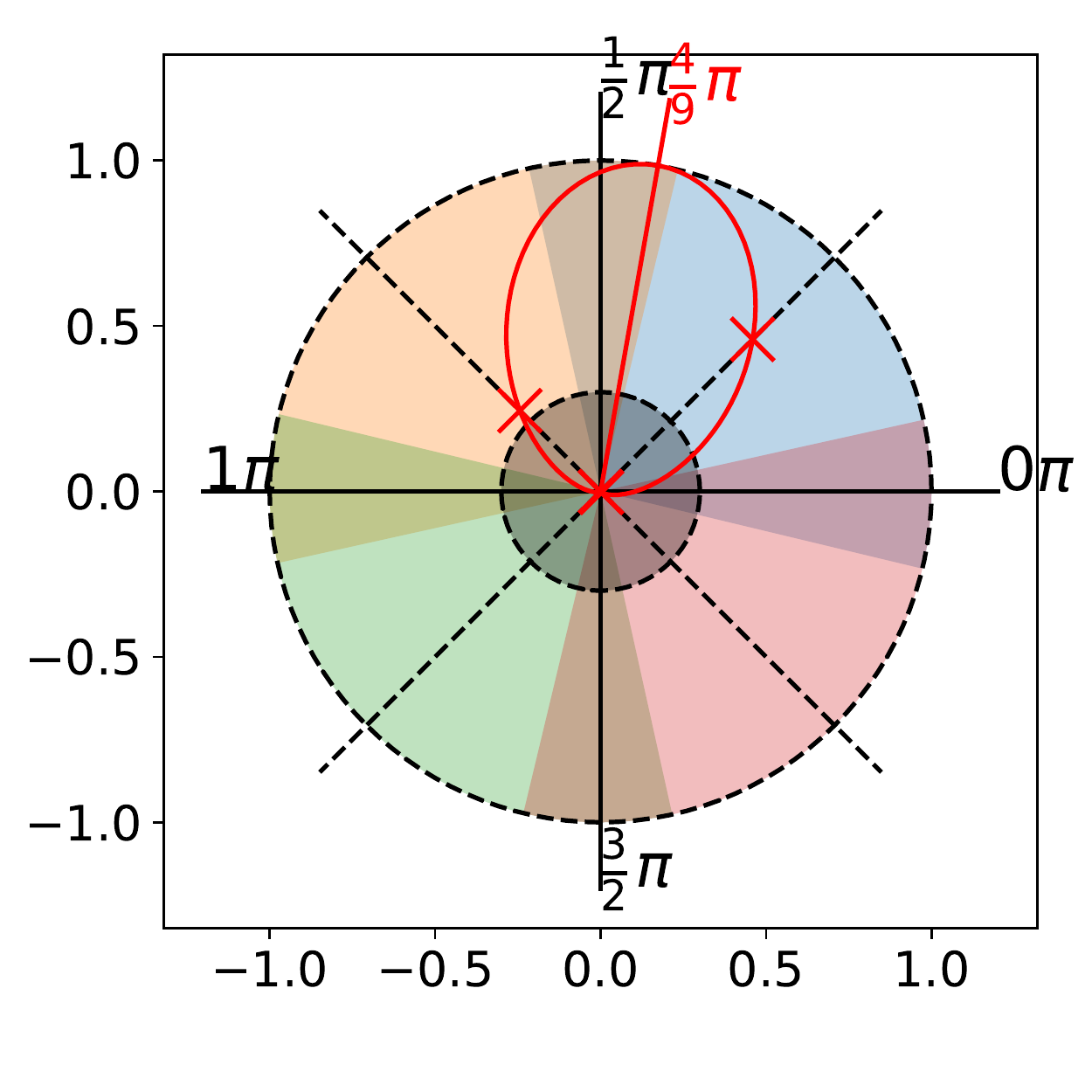}
\label{fig:cn_angle2code_2}
}\\
\subfloat[Vector not in any overlap region ($\tau$ equals to $0.1$)]{
\includegraphics[width=0.45\linewidth]{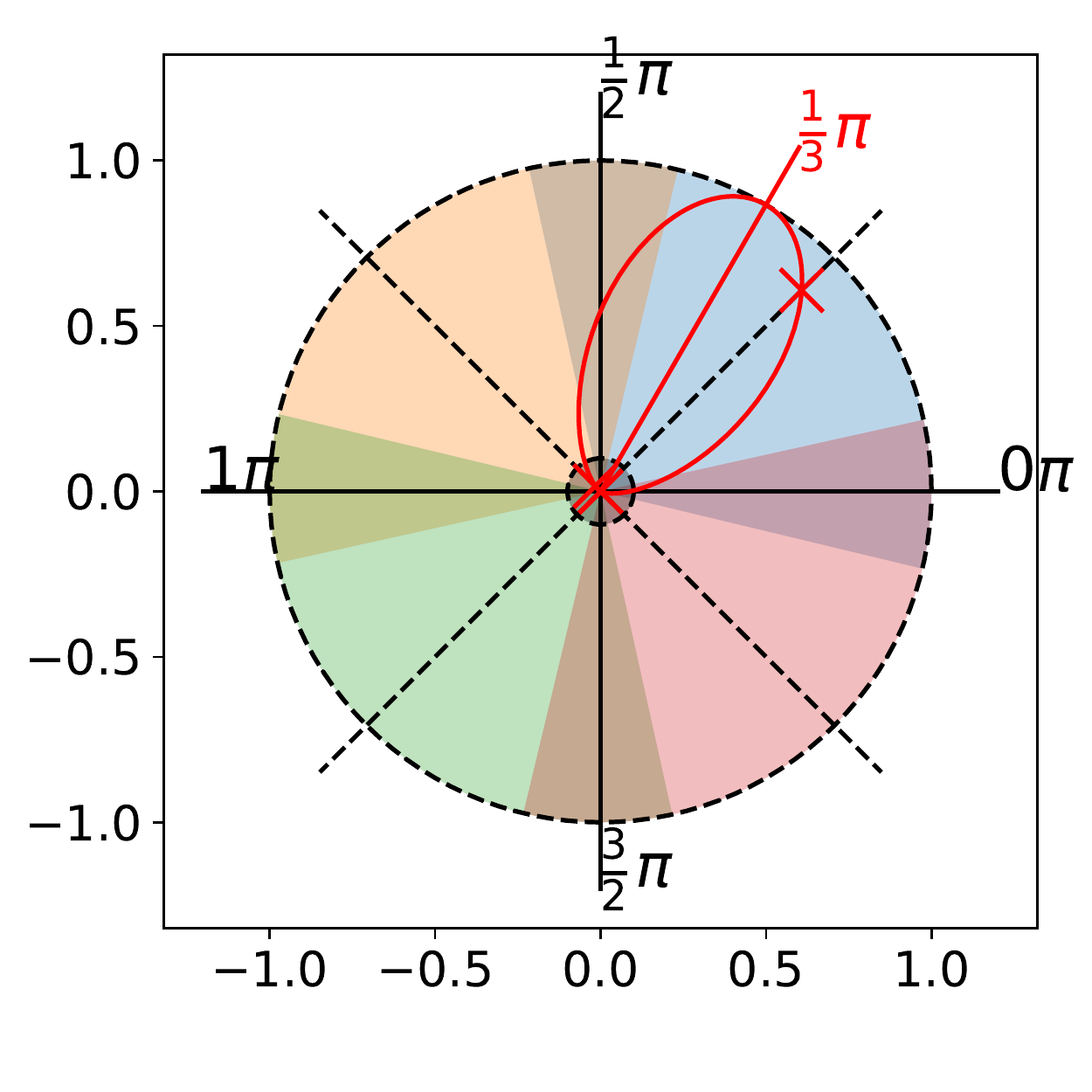}
\label{fig:cn_angle2code_3}
}\hspace{0.02\linewidth}
\subfloat[Vector in an overlap region ($\tau$ equals to $0.1$)]{
\includegraphics[width=0.45\linewidth]{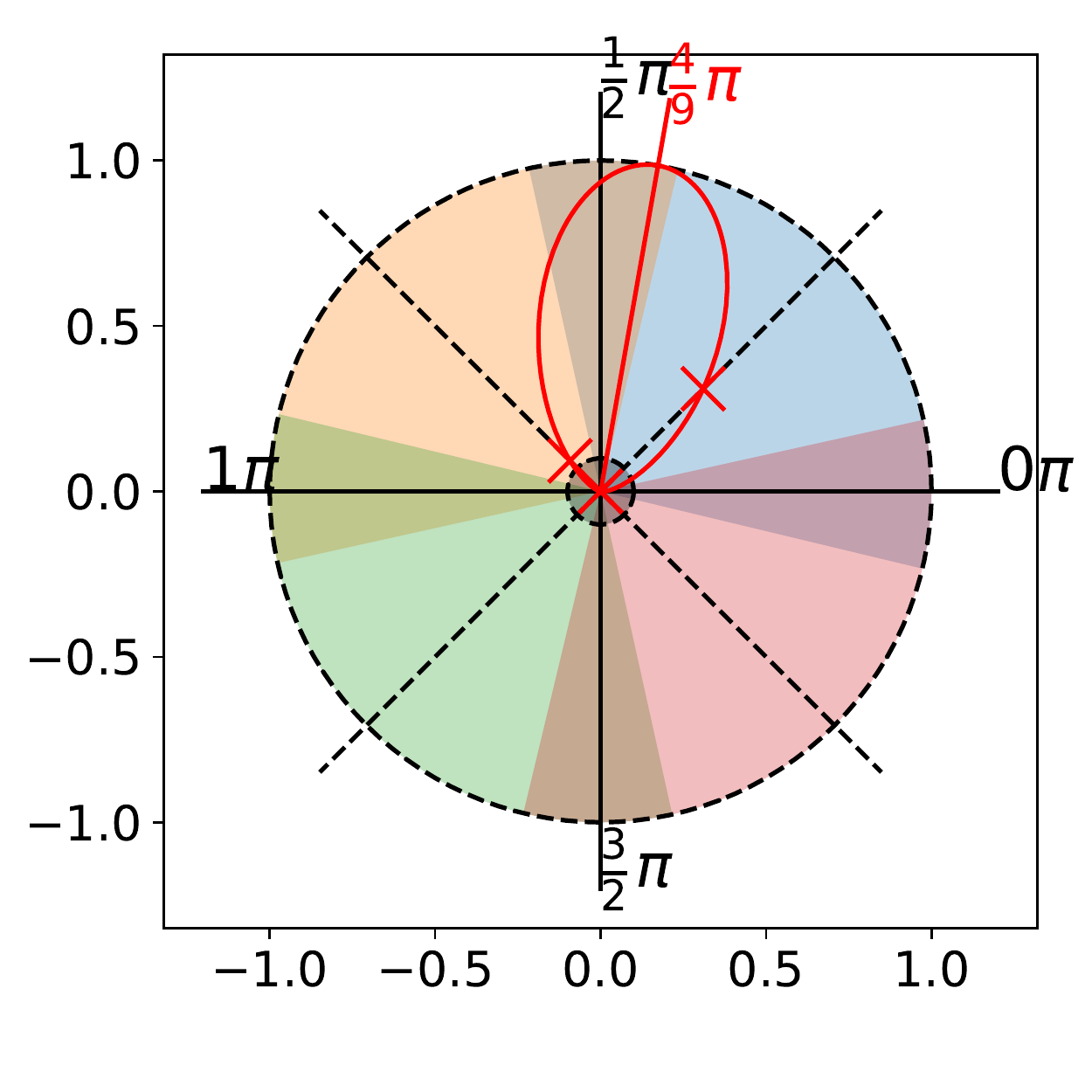}
\label{fig:cn_angle2code_4}
}
\caption[Illustration of Partitioning in the Stacked Autoencoder]{The red straight line denotes the angle $\tilde{\theta}$, with the red curve indicating the amount of decay from the maximum $1$ to the minimum $0$ while moving away from $\tilde{\theta}$. The red cross maker demonstrates a value on the midpoint, with in the brighter area telling it is above the selection threshold while in the darker area telling below the selection threshold. } \label{fig:cn_angle2code}
\end{figure}

\subsection{Consideration of Model Uncertainty}
The soft code $\bar{c}$ utilizes the angular proximity of input data in a 2D representation. However, partitioning based on the soft code alone does not always imply the trained model is more specialized. The predictive uncertainty of a trained model with respect to the input data is also indicative of how much the model has ``specialized'' on the data. Intuitively, if a model is specialized on one partition of the input space, it should have a lower predictive uncertainty on the prediction of the data in the partition, but higher uncertainty on other data. In \cite{fang2019teamnet}, we found that the
entropy computed from the softmax output of a neural network model is a good
surrogate for the uncertainty of the model on the data.  Formally, we denote
$H(\hat{y}_k|x, \phi_k)$ the entropy of the $k^{th}$ submodel parameterized by
$\phi_k$ with respect to the input $x$,
\begin{equation}
H(\hat{y}|x, \phi_k) = - \sum_c p(\hat{y}=c|x, \phi_k) \log p(\hat{y}=c|x,
\phi_k),
\end{equation}
where $p(\hat{y}=c|x, \phi_k)$ is the predictive probability of output $c =
1,2,...,C$ for input $x$ from submodel $k$. 

To encourage the assignment of $x$ to a submodel that has the lowest predictive
uncertainty, we introduce a $K$-dimension vector $\bar{\bar{c}}$ as follows:   

\begin{equation}
\bar{\bar{c}} = [\bar{\bar{c}}_1 \ldots \bar{\bar{c}}_K], \quad \bar{\bar{c}}_i=
\begin{dcases}
\tau & i = \argmin_k H(\hat{y}_k|x, \phi_k) \\
0 & otherwise. \\
\end{dcases}
\end{equation}

Clearly, $\bar{\bar{c}}$ is a one-hot vector scaled by $\tau$. 
\subsection{Partition of Input Data}
To this end, we have obtained two $K$-dimension codes $\bar{c}$ and $\bar{\bar{c}}$ for each input data $x$. To decide the final partition of input data, we should take both
into account. This can done by a simple linear combination: 
\begin{equation}
c = \alpha\bar{c} + (1-\alpha)\bar{\bar{c}}.
\end{equation}
In the experiments, we set $\alpha = \half$.  

Let $\mathcal{P}(x) = \left\{k| c_k \ge \frac{\tau}{2}\right\}$ denote the
indices of partitions (submodels) that input $x$ contributes to. Clearly,
$\mathcal{P}(x)$ cannot be an empty set since its respective $\bar{\bar{c}}$
contains one element that equals to $\tau$. In the case that the cardinality of
$\mathcal{P}(x)$ is greater than one, this implies that $x$ will be used to
train multiple submodels. 
\subsection{Neural Network Generator}
The architecture of the generator network is illustrated in
Figure~\ref{fig:cn_generator}. A neural network generator $G$ takes an element $k$ in $\mathcal{P}(x)$ (being converted to a one-hot vector) as input and generates the parameters $\phi_k$ of the $k$th submodel.  CacheNet is agnostic to the target neural network architecture, which is decided by the target application. For example, for image classification, Shake-Shake~\cite{gastaldi2017shake} has been shown to perform well across several datasets. Given $K$, we scale down the target neural network architecture to have reduced capacity. 

Suppose $\hat{y}_k$ is the prediction of the $k$th submodel for $x$, noted by $\hat{y}_k = F(x;\phi_k)$. To avoid overfitting, we allow parameter sharing across the submodels. The proportion of parameters to be shared, the depth and the width of the shared networks are hyper-parameters to be determined by the neural network
structure of the submodels. For an input data $x$ and its label $y$, we first
compute $\mathcal{P}(x)$.  The cross-entropy loss for classification is given by, \begin{equation} \label{eq:j_f}
J_F(x,y) = \sum_{k\in \mathcal{P}(x)}H(\hat{y}_k,y)
\end{equation}
Minimizing $E_{p_D(x)}J_F(x,y)$ leads to a more accurate prediction with respect to the dataset. 

\begin{figure}[ht]
\centering
\includegraphics[width=0.9\linewidth]{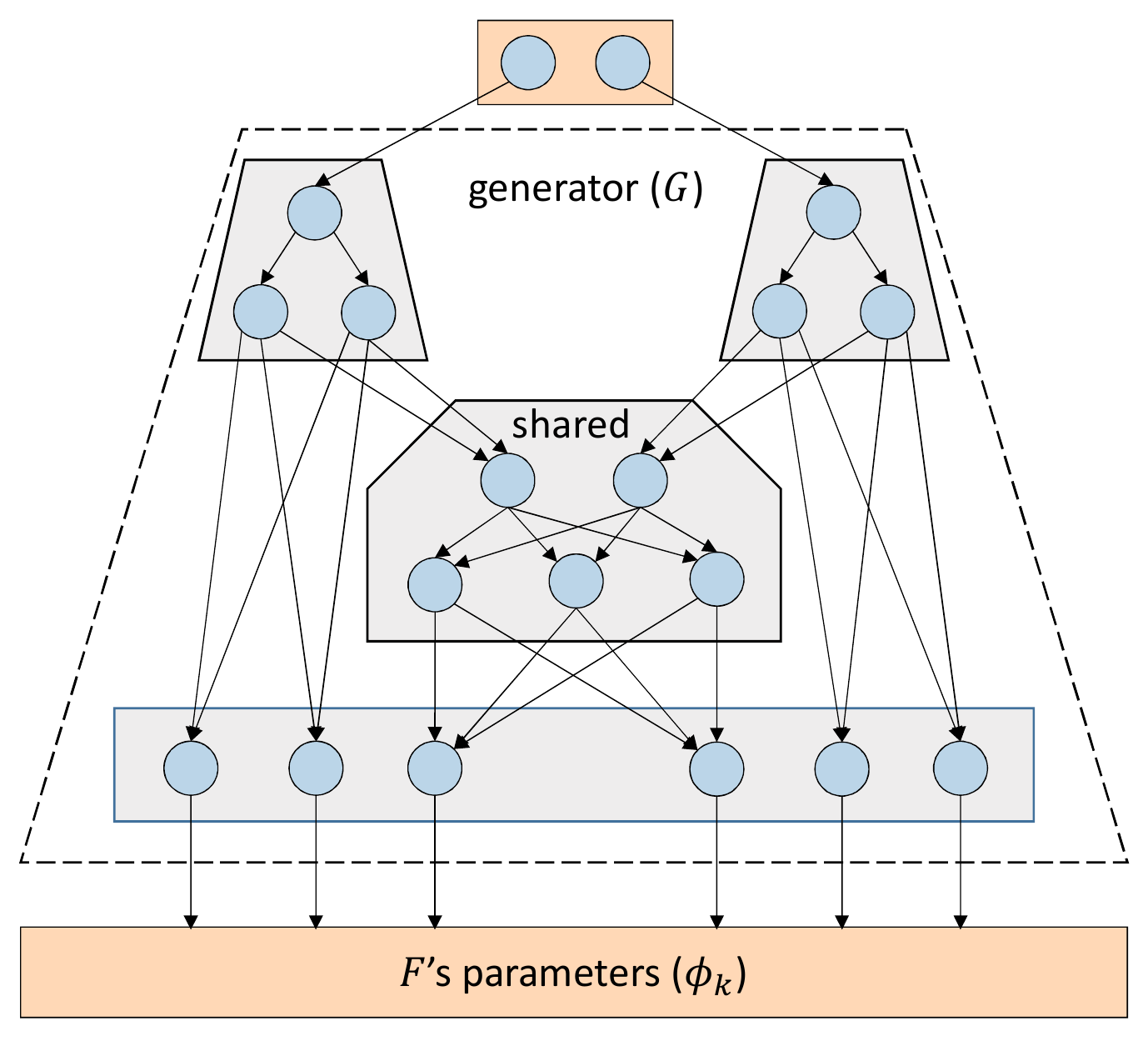}
\caption[Data Flow in CacheNet's Generator]{The generator $G$ takes a one-hot vector $\delta_i$ as input and generates the parameters of the $i^{th}$ partition. Values (either $0$ or $1$) of each dimension in $\delta_i$ are used to deactivate or activate a corresponding branch.}
\label{fig:cn_generator}
\end{figure}
\subsection{Training Algorithm}
In CachNet, there are three networks that need to be trained, namely, the
stacked encoder, the stacked decoder and the generator network. Since the
output of the stacked encoder contributes to the input of the generator
network, they need to be trained jointly. 

The lower-dimension representation $\bar{z}$ is the most informative of a particular input $x$ if two lower bounds $\mathcal{L}^*(\xi, \psi)$ and
$\mathcal{\bar{L}}^*(\bar{\xi}, \bar{\psi})$ are maximized, and a submodel's predictions are the most accurate if $E_{p_D(x)}J_F(x,y)$ is minimized. Thus, the minimization objective $J$ should be $E_{p_D(x)}J_F(x,y)$ added to the negation of $\mathcal{L}^*(\xi, \psi)$ and $\mathcal{\bar{L}}^*(\bar{\xi}, \bar{\psi})$:
\begin{equation} \label{eq:j}
J = E_{p_D(x)}J_F(x,y) - \mathcal{L}^*(\xi, \psi) - \mathcal{\bar{L}}^*(\bar{\xi}, \bar{\psi}).
\end{equation}

To better converge, $E_{p_D(x)}J_F(x,y)$, $\mathcal{L}^*(\xi, \psi)$, and $\mathcal{\bar{L}}^*(\bar{\xi}, \bar{\psi})$ are optimized stage-wisely and batch-wisely. Let $J^{(i)}$ be $J$ with respect to a batch $(X^{(i)}, Y^{(i)})$ drawn from the dataset $D$. Suppose the generator $G$ is parameterized by $\chi$, and $\kappa$ is the set
of $\{\xi, \psi, \bar{\xi}, \bar{\psi}, \chi\}$. The training algorithm should iteratively apply gradient updates to $\kappa$ (or $\chi$) with respect to the loss
function $J^{(i)}$ and descend to a minimum of $J$ (as shown in Algorithm~\ref{alg:train}).

\begin{algorithm}
  \caption{Training CacheNet}\label{alg:train}
  \begin{algorithmic}[1]
  \LeftComment\scalebox{0.9}{let $\eta$ be the learning rate}
  \LeftComment\scalebox{0.9}{let $\nu$ be the epoch stopping gradient updates in $\xi, \psi, \bar{\xi}, \bar{\psi}$}
    \Procedure{Train}{$\eta,\nu$}
      \While{$J^{(i)}$ is decreasing}
        \State draw the next batch $(X^{(i)}, Y^{(i)})$ from $D$
        \If{\#epoch $< \nu$}
            \State \scalebox{1.2}{$ \kappa \gets \kappa - \eta  \nabla_{\kappa} J^{(i)}$} 
        \Else
            \State \scalebox{1.2}{$ \chi \gets \chi - \eta  \nabla_{\chi} J^{(i)}$} 
        \EndIf
      \EndWhile
    \EndProcedure
  \end{algorithmic}
\end{algorithm}

\section{CacheNet Inference}
\label{sect:cn_inferencing}
With CacheNet, inference on end devices is accelerated by caching submodels of lower computation complexity. Depending on storage availability, one or multiple submodels can be stored on end devices. At any time, only one submodel is {\it
active} and is used to make predictions.  Given an input data sample $x$, the
active submodel $k$ outputs $\hat{y}$, the label of $x$ and the predictive
entropy $H(\hat{y}|x,\phi_k)$.  If the entropy is above a certain threshold,
$\hat{y}$ will be returned. Otherwise, two situations may arise, i) $x$ is
better handled by another cached submodel, and ii) $x$ is better handled by a
submodel not in cache. The latter case is called a {\it cache miss}.  Like
caching in memory hierarchy, CacheNet needs to handle cache misses by replacing
an cached ``item'' (model).  However, unique to CacheNet, the newly cached
``item'' is not the input data but a suitable model. 
\subsection{Submodel Selection}
In Section~\ref{sect:cn_training}, a $K$-dimension code $\bar{c}$ is computed for each input data sample using S-InfoVAE and the subsequent mapping in polar coordinates. In the training stage, $\bar{c}$ contributes to the input to the generator network that generates the parameters of respective submodels. In the inference stage, $\bar{c}$ can be used to select the submodel to make prediction given an input data sample. In particular, the joint optimization of S-InfoVAE, generator network and submodels aligns the output of S-InfoVAE with the submodel that has lowest predictive uncertainty. Thus, we can simply select the submodel whose index corresponds to the largest element in $\bar{c}$. Note in the inference stage, we do not need to calculate the predictive uncertainty for each submodel. Instead, 
only one submodel is applied. This is one of the key differences between
CacheNet and the work in \cite{fang2019teamnet}. S-InfoVAE can be executed on the end device or on
the edge server. In the former case, extra storage and computation overhead are
introduced. In the latter case, submodel storage and selection are delegated to the edge server.   
\subsection{Cache Replacement}
\label{subsect:cn_cache_replacement}
When the predictive entropy is below a preconfigured threshold using the active submodel, the input data $x$ is sent to the edge server, which will perform inference on behalf of the end device. Additionally, by submodel selection, the edge server determines a suitable model for $x$. A cache miss occurs on the end device. The newly selected submodel will be downloaded to the device to replace an existing model. Here, we adopt the Least Recently Used (LRU) policy and select the model that is least recently used. By the virtue of LRU, such a policy does not suffer from B\'{e}l\'{a}dy's anomaly.  In other words, as the cache size increases, the cache miss rate does not increase. 

\section{Evaluation}
\label{sect:cn_evaluation}

In this section, we evaluate CacheNet with two different real-world datasets (the CIFAR-10 \cite{krizhevsky2009learning} and the Frontal View Gait (FVG) dataset \cite{gait-recognition-via-disentangled-representation-learning}), and test CacheNet's performance with respectively two different neural network models (Shake-Shake \cite{gastaldi2017shake} and ResNet \cite{he2016deep}). 

\subsection{Datasets}

\paragraph*{CIFAR-10} 
CIFAR-10 \cite{krizhevsky2009learning} is a benchmark dataset for image classification, comprised of $60,000$, $32\times32$ colored images and $10$ classes (such as automobile, bird and horse) in total. Although CIFAR-10 is not a video dataset and is an image classification dataset, image classification is still a valid scenario if it is in a video processing pipeline (e.g. where the background has been removed previously from the video). In this case, temporal locality still applies while consecutive images would be less redundant owing to the earlier steps in the pipeline. For example, a horse (possibly shot in different angles with different scales) in the video is still likely to appear multiple times in the sequence, even when the background has been removed (e.g. object detection). 

For fair evaluation, test images are not supposed to be seen during training. Thus, we set aside $10,000$ images for testing. To simulate  temporal locality in a video pipeline, the synthesized image sequence in testing is composed of a sample of the $10,000$ images in the way that images with the same label are concatenated together. 

To reduce overfitting, data augmentation techniques are used, including: 1) random cropping and 2) random flipping. Apart from data augmentation, Shake-Shake regularization has been applied to reduce overfitting \cite{gastaldi2017shake}, and batch normalization to reduce internal covariate shift \cite{ioffe2015batch}.

\begin{figure*}[ht] 
\centering
\subfloat[Partition A]{
\includegraphics[width=0.23\linewidth]{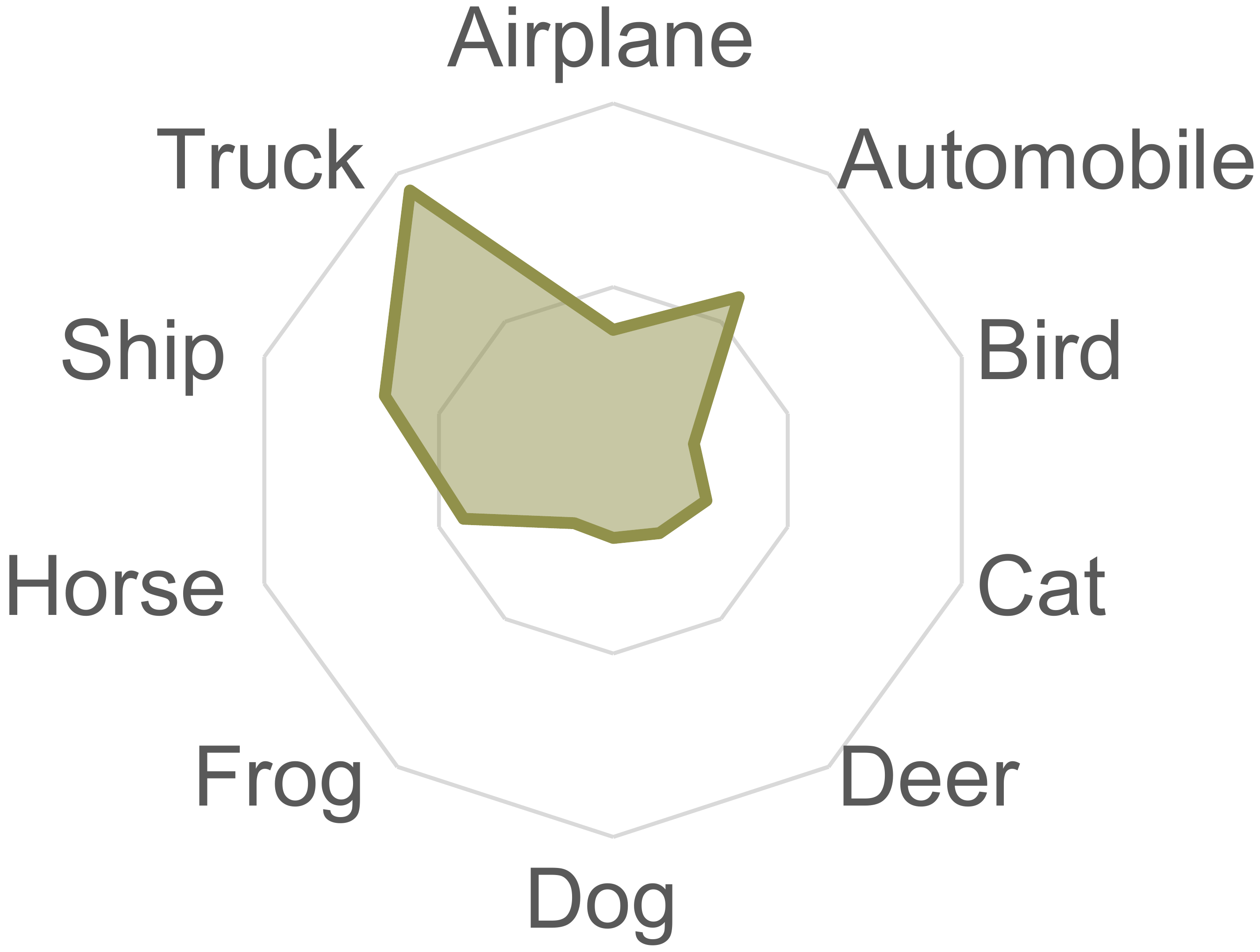}
\label{fig:cn_cifar_vae_part_a}
}
\subfloat[Partition B]{
\includegraphics[width=0.23\linewidth]{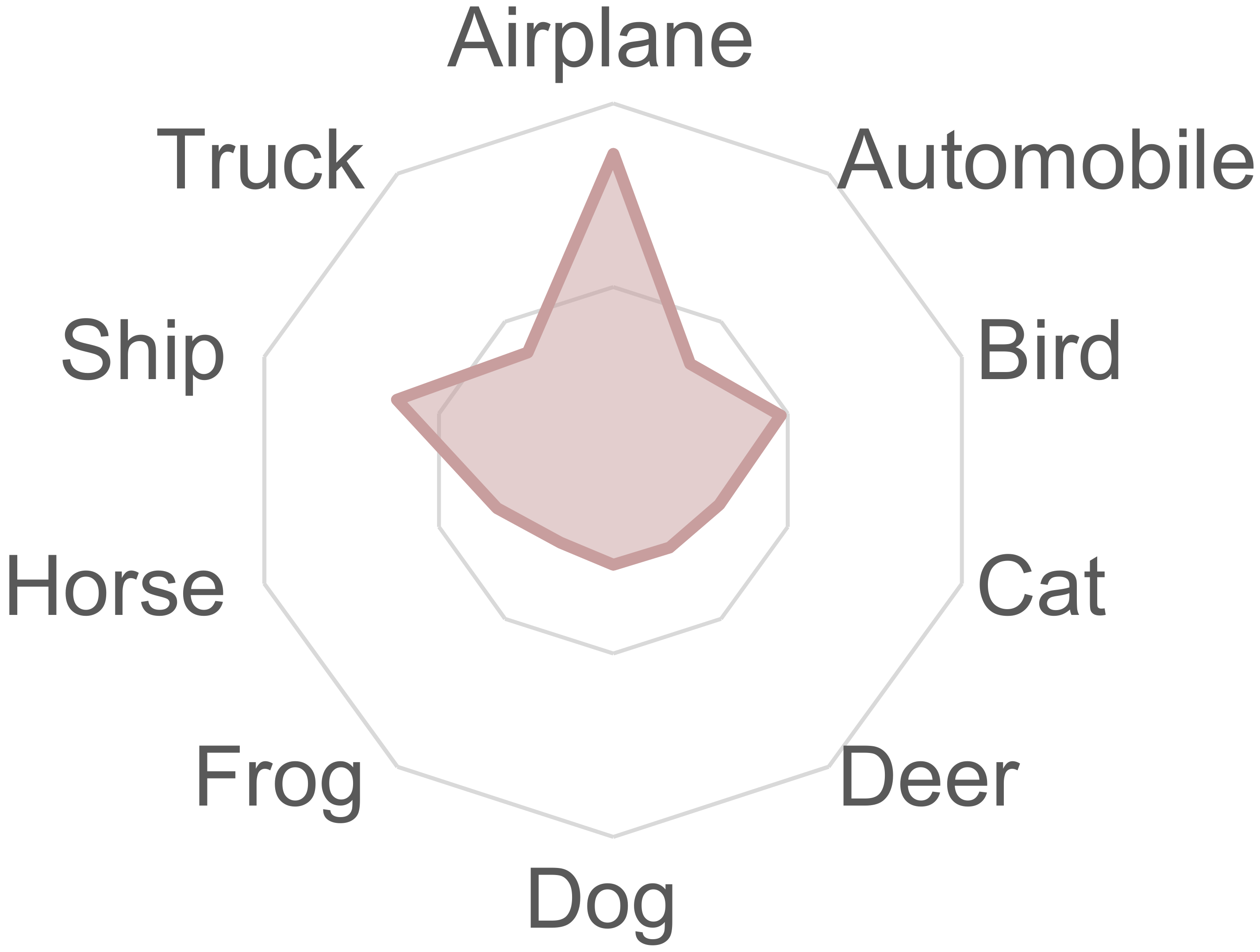}
\label{fig:cn_cifar_vae_part_b}
} 
\subfloat[Partition C]{
\includegraphics[width=0.23\linewidth]{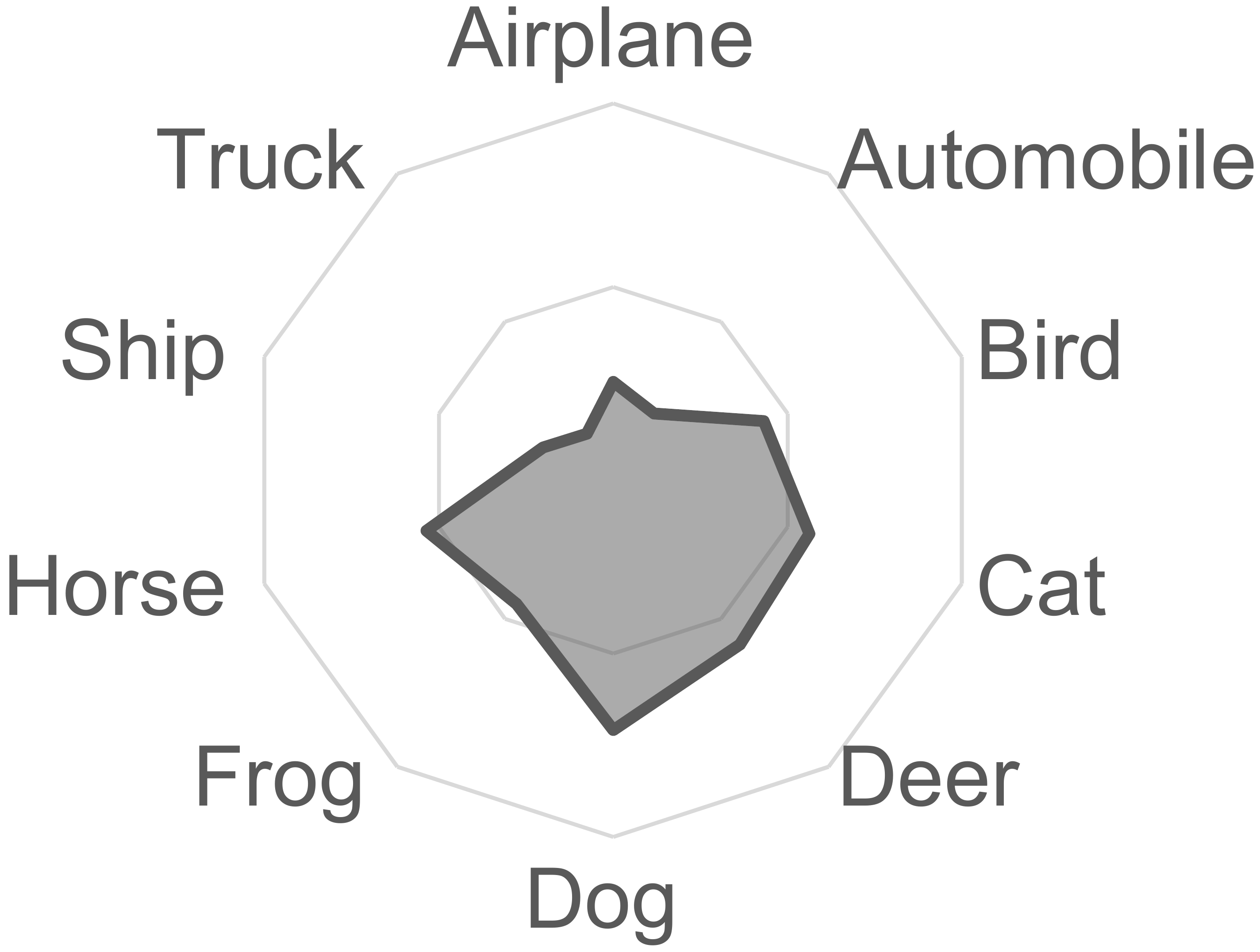}
\label{fig:cn_cifar_vae_part_c}
}
\subfloat[Partition D]{
\includegraphics[width=0.23\linewidth]{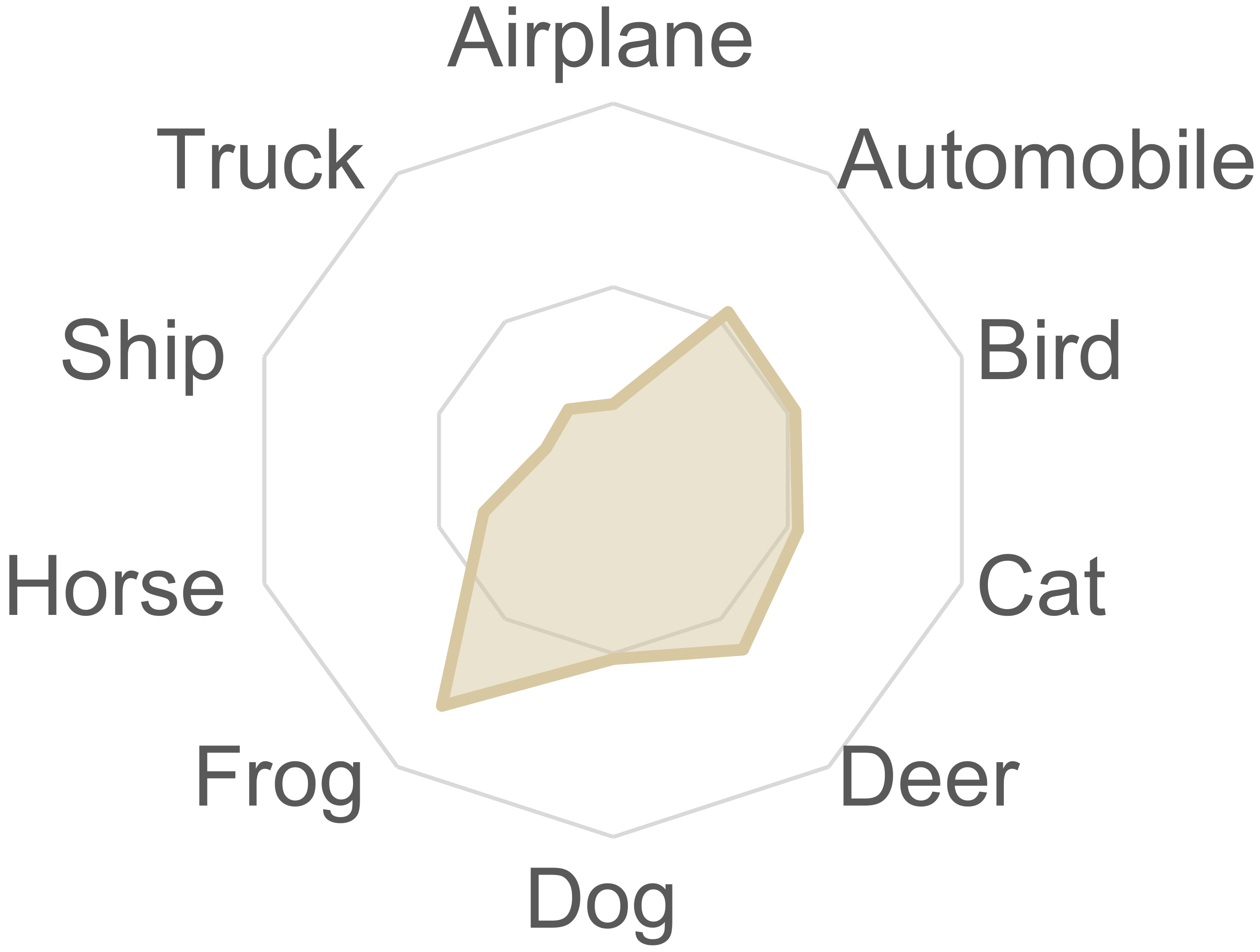}
\label{fig:cn_cifar_vae_part_d}
}
\caption[CacheNet's Specialization in CIFAR-10]{For CIFAR-10, partition A is more specialized in trucks and automobiles; partition B can predict airplanes and ships better; partition C is more certain of the horse, dog, and cat classes; partition D knows more about frogs and deer.}
\label{fig:cn_cifar_vae}
\end{figure*}

\begin{figure*}[ht] 
\centering
\subfloat[Partition A]{
\includegraphics[width=0.23\linewidth]{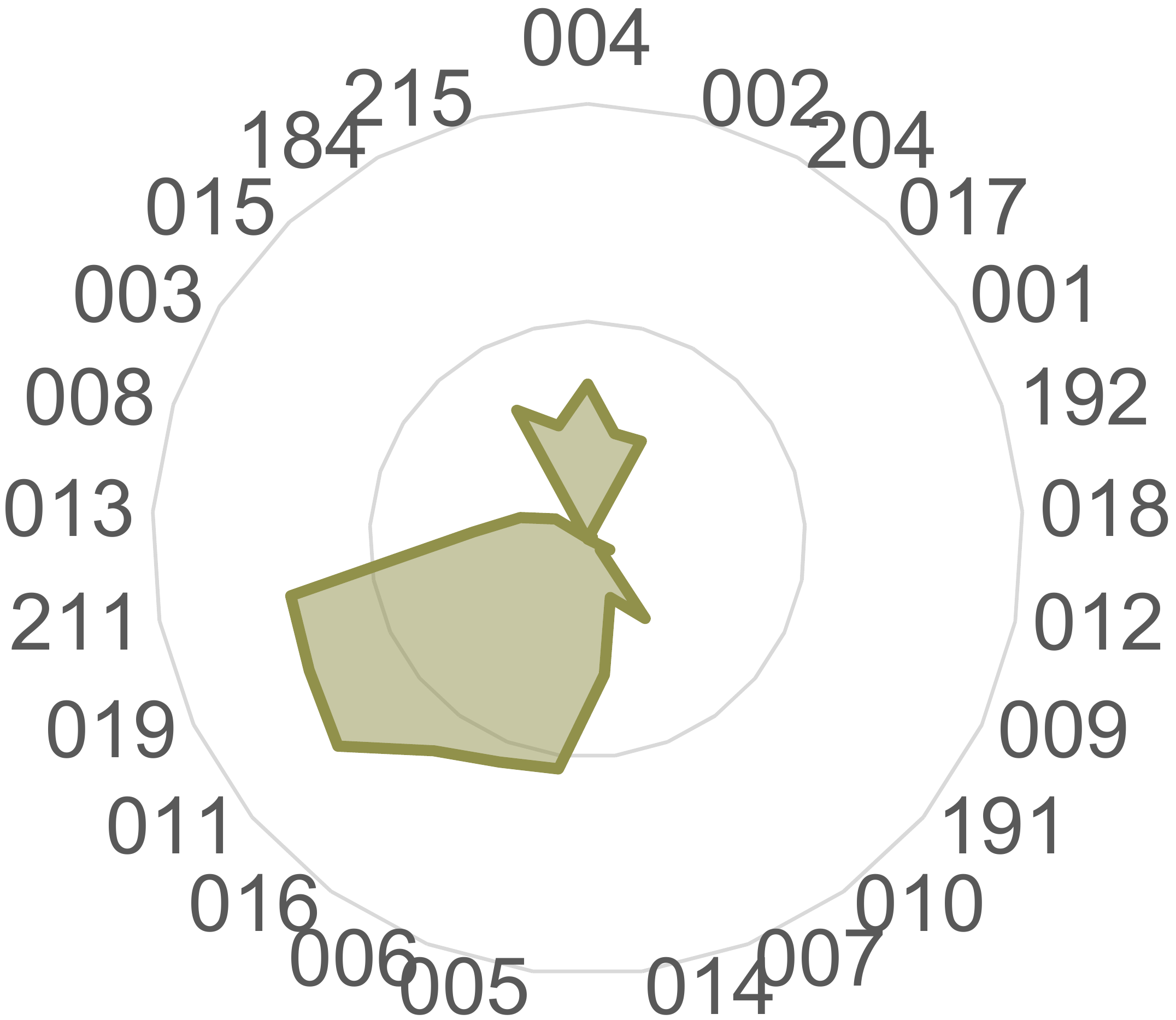}
\label{fig:cn_fvg_vae_part_a}
}
\subfloat[Partition B]{
\includegraphics[width=0.23\linewidth]{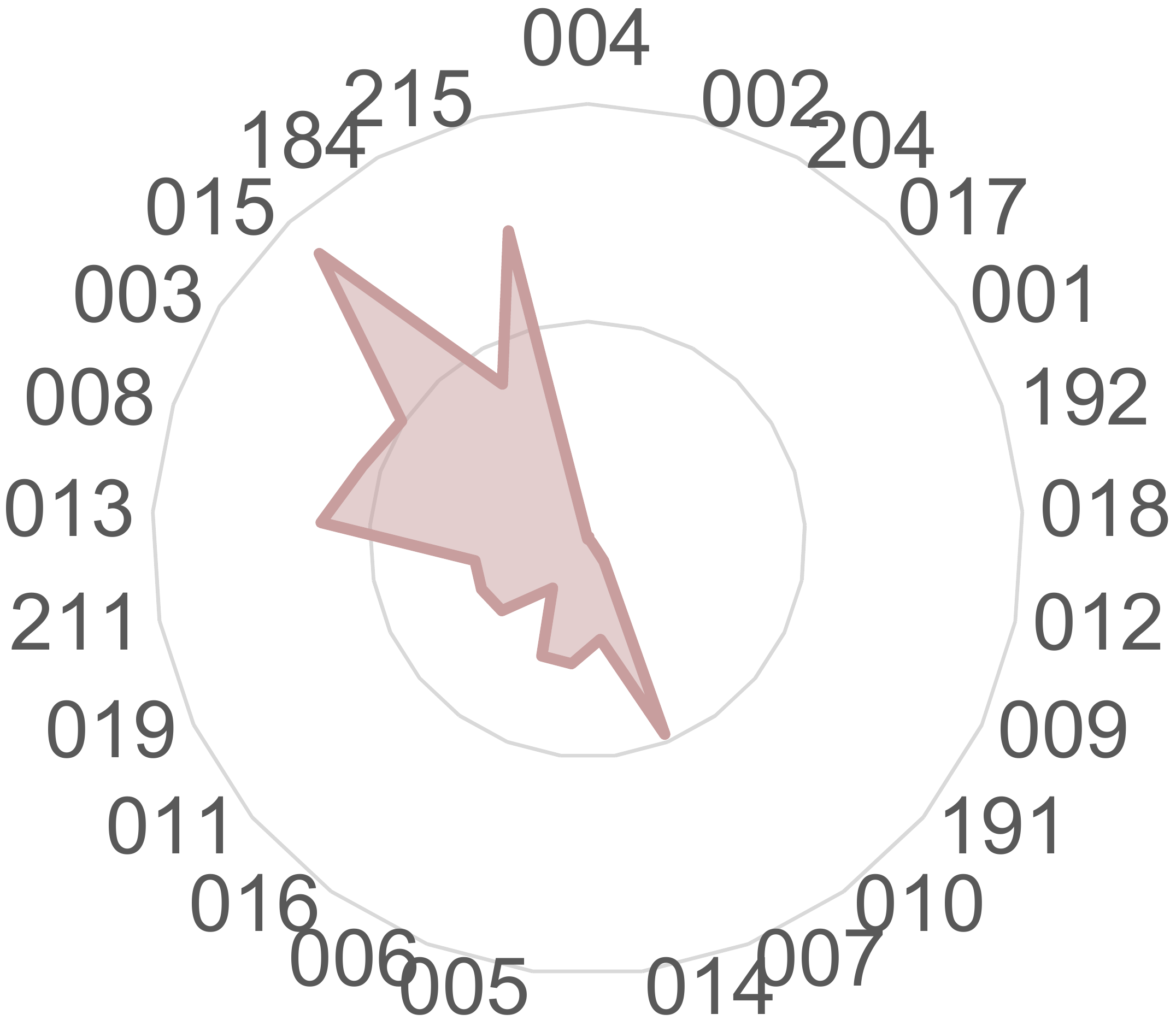}
\label{fig:cn_fvg_vae_part_b}
} 
\subfloat[Partition C]{
\includegraphics[width=0.23\linewidth]{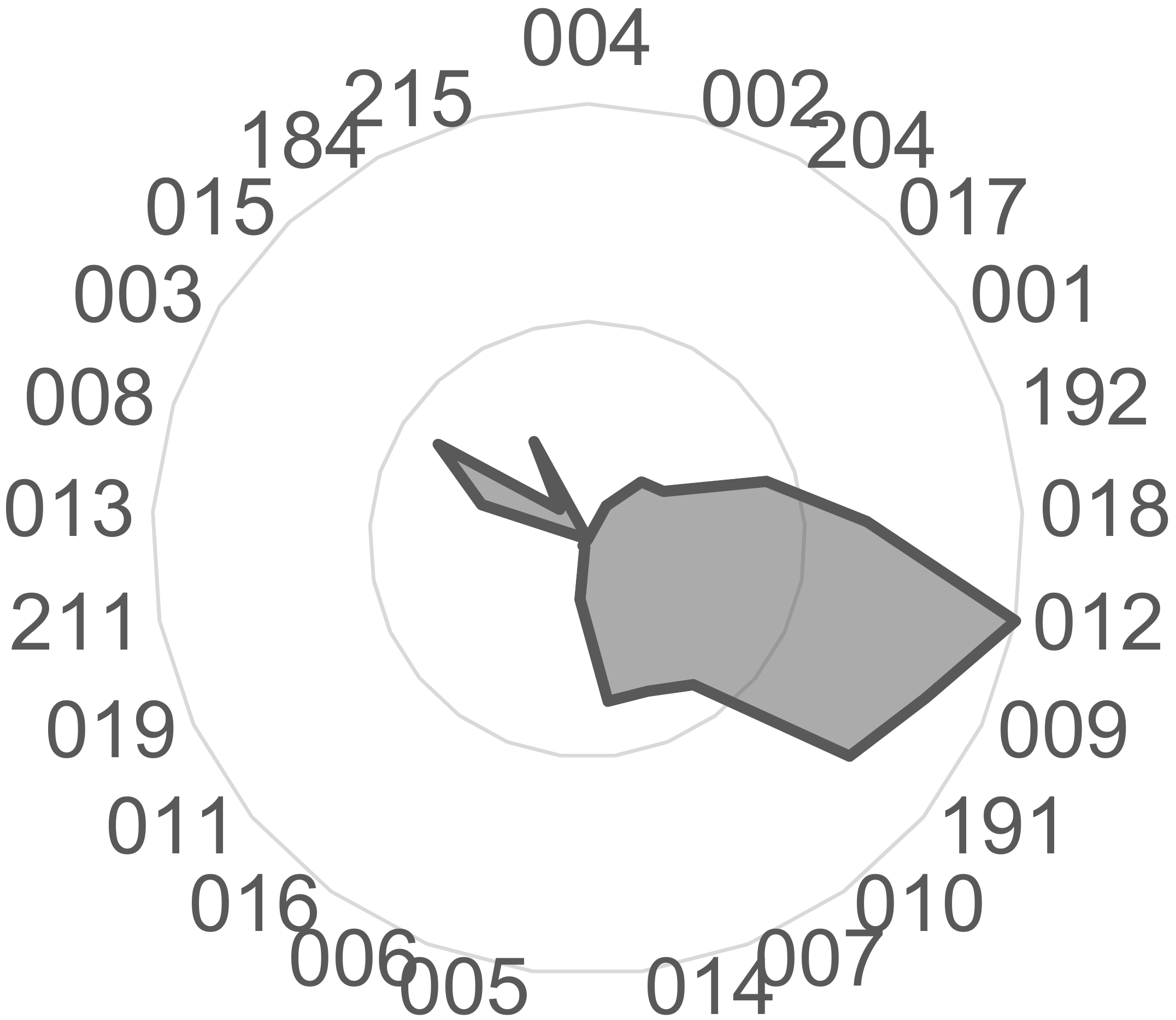}
\label{fig:cn_fvg_vae_part_c}
}
\subfloat[Partition D]{
\includegraphics[width=0.23\linewidth]{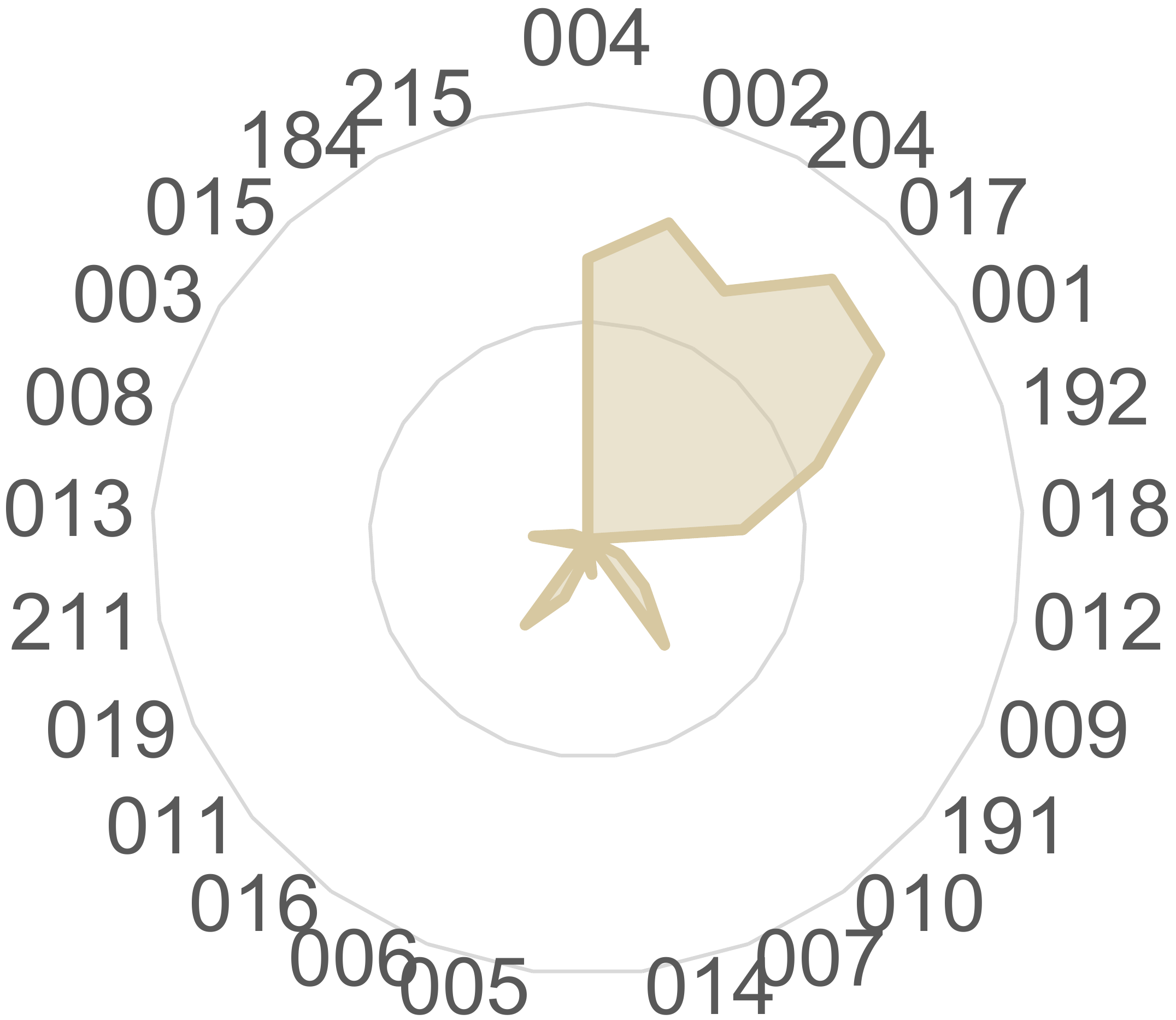}
\label{fig:cn_fvg_vae_part_d}
}
\caption[CacheNet's Specialization in FVG]{For FVG, partition A is more certain of person identifier (PID) $211$, $019$, $011$, $016$, $006$, and $005$; partition B is specialized in PID $013$, $008$, $003$, $015$, and $215$; partition C knows more about PID $010$, $191$, $009$, $012$, and $018$; partition D is more certain of PID $004$, $002$, $204$, $017$, and $001$.}
\label{fig:cn_fvg_vae}
\end{figure*}

\paragraph*{FVG} 
FVG is a person re-identification dataset, first introduced in \cite{gait-recognition-via-disentangled-representation-learning}, as a collection of frontal walking videos from 226 subjects. In total it contains $2,85$6 videos at $15$ frames per second with a resolution of $1920\times1080$. 

In contrast to other person re-identification datasets in surveillance settings, FVG is the first to focus on the frontal view. This makes it useful for two reasons: (i) It contains temporal locality in the form of a fixed background and the same subject walking towards the camera, which can be leveraged for caching. 
(ii) Having a frontal view means that it contains minimal gait cues.

To reduce the chance of overfitting and improve generalization ability we use data augmentation techniques \cite{shorten2019survey} on this dataset as well. We first oversample the images by interpolating between existing frames. This technique preserves the extrinsic distribution while allowing us to experiment with cache performance by varying the degree of temporal locality. Additionally, in the original dataset the average frame rate of each video is $15$ frames per second. That is only half of the frame rate of a HD video (generally $30$-$60$ frames per second). Since each video sample is of the subject walking straight towards the camera from a distance, it contains intrinsic depth information that can be utilized to synthesize intermediate frames. As such, we use DAIN \cite{bao2019depth}, a state of the art approach that leverages the depth information to interpolate between the frames.

\subsection{Experimental Setup}

CacheNet's performance is evaluated on two different datasets (CIFAR-10 and FVG), three end devices (Jetson TX2, Jetson Nano, and Raspberry Pi 4) and two deep learning frameworks (NCNN and TensorFlow Lite). There are two baselines to compare with: a) running a full model (Shake-Shake-26 or ResNet-50) on an end devices ({\it Device}), and b) offloading the full model onto an edge server ({\it Edge}). Different thresholds are evaluated to better trade off hit rate against accuracy: for CIFAR-10, they are $0.5$, $0.6$, $0.7$, $0.75$, and $0.8$; for FVG, they are $1.5$, $2.0$, $2.3$, $2.5$, and $2.7$. (Here, A larger threshold in FVG is caused by more classes (neurons) at the output layers.) Furthermore, on the FVG dataset, we evaluate two video frame rates 15 FPS and 30 FPS (at inference) with both trained at 60 FPS (by using data augmentation).

The number of submodels $K$ is set to $4$ in the experiment. For a possible convergence, CacheNet is trained on TensorFlow with $4$ NVIDIA 1080TI graphic cards. Per CIFAR-10, CacheNet partitions Shake-Shake-26 (with 26 layers) into $4$ Shake-Shake-8 (with 8 layers) neural network submodels for caching; per FVG, CacheNet partitions ResNet-50 into $4$ ResNet-20 (but with fewer channels per layer). 

CacheNet's inference is distributed between the edge server and the end device in the experiment. One submodel is cached and runs on the end device, while submodel storage and selection are delegated to the edge server. 
End devices are evaluated with limited storage to mimic that of end devices such as security cameras. 
One Intel Xeon CPU core is enabled on the edge server to representatively simulate those of most of WiFi access points (e.g. a 500 megahertz MIPS processor on the Arlo SmartHub) with generally limited compute power. There is sufficient storage on the edge server comparable to that of WiFi access points (e.g. a 128 gigabyte SD card on the eufy HomeBase and a 2 terabyte USB hard drive onto the Arlo SmartHub). End devices are connected to the edge server through a WiFi router, via WiFi 5G (802.11ac) and an Ethernet cable, respectively.

TensorFlow submodels from training were converted to NCNN and TensorFlow Lite submodels and stored on the edge server. Whenever a submodel is needed, the end device initiates an HTTP/1.1 request to the edge server, and then the chosen submodel on the edge server is encoded in an HTTP/1.1 and protobuf message then sent back to the end device. OpenCV is also used in the experiment to read a testing image sequence (video) into the memory and convert them into tensors.



\subsection{Results}

\paragraph*{Specialization} 
Specialization is crucial for caching because a non-specialized partition cannot match the full model's performance by any chance even for a smaller subset of input. There are two aspects we would investigate: (a) whether similar input images are mapped to the same partition; (b) whether input images are partitioned roughly evenly to fully utilize the capacities of all submodels, considering both CIFAR-10 and FVG are approximately balanced datasets. 

Figure~\ref{fig:cn_cifar_vae} and \ref{fig:cn_fvg_vae} illustrate the number of input images per class being mapped (by S-InfoVAE) to a particular partition. They answer most of our concerns: (a) A partition roughly covers most of similar input images from the same class. e.g. for CIFAR-10, partition A is more specialized in trucks and automobiles; partition B knows better airplanes and ships; for FVG, partition A is more certain of person identifier (PID) $211$, $019$, $011$, $016$, $006$, and $005$; and partition B is specialized in PID $013$, $008$, $003$, $015$, and $215$. (b) In both cases of CIFAR-10 or FVG, the areas (Figure~\ref{fig:cn_cifar_vae} and \ref{fig:cn_fvg_vae}) that partitions occupy are roughly even. It implies the total number of (image) instances they span are approximately the same. 

\paragraph*{Convergence} 
Not all neural networks converge. Thus, whether CacheNet is useful depends on whether it converges or not per the particular dataset. In CIFAR-10 and FVG, we can see  (Figure~\ref{fig:cn_loss}) that their losses both start high but converge closer and closer to zero. Since FVG is a smaller dataset compared to CIFAR-10, CacheNet with FVG converges faster (in fewer iterations) than CIFAR-10.

\begin{figure}[t]
\centering
\subfloat[CIFAR-10]{
\includegraphics[width=0.45\linewidth]{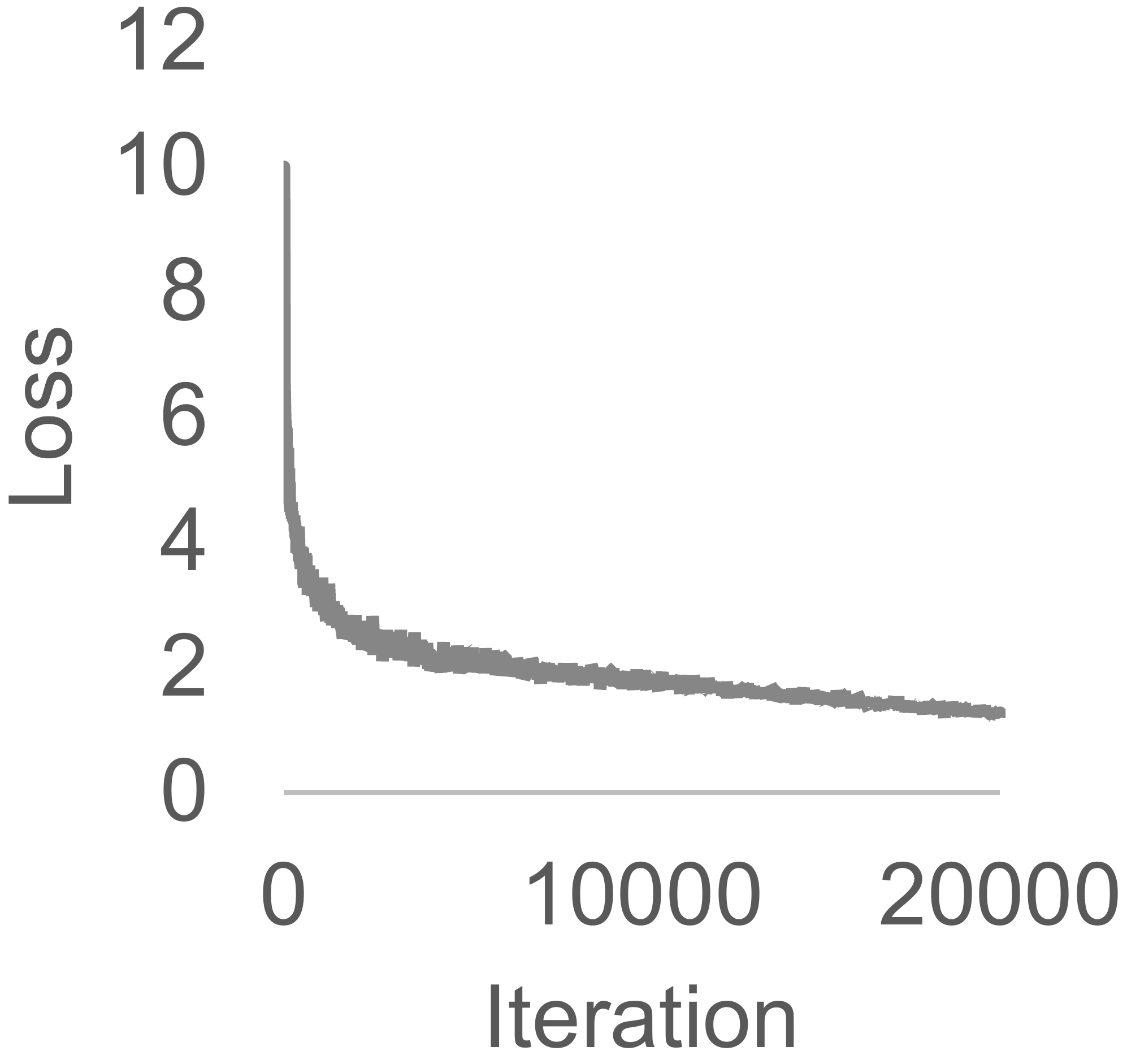}
\label{fig:cn_cifar_loss}
}
\subfloat[FVG]{
\includegraphics[width=0.45\linewidth]{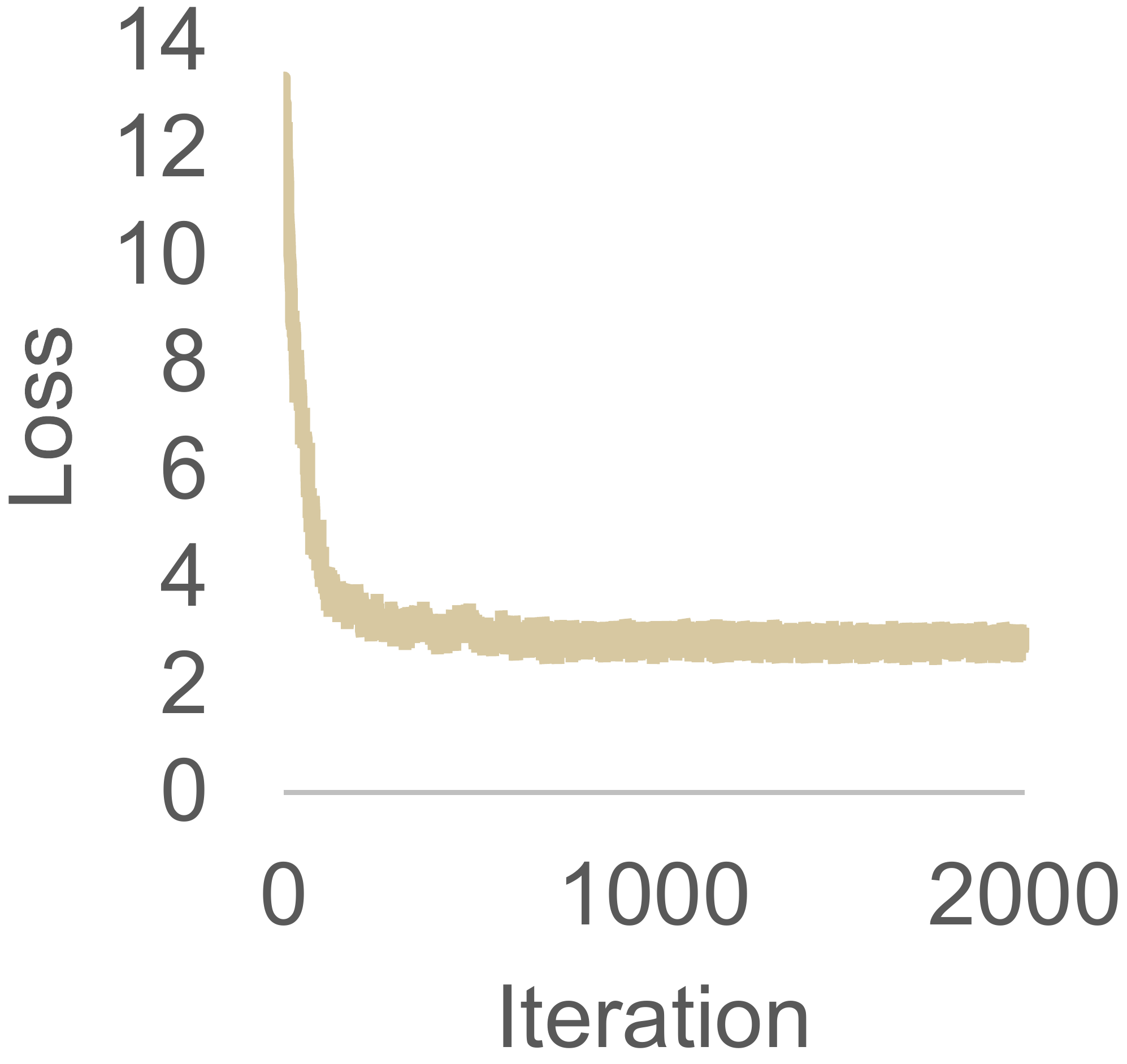}
\label{fig:cn_fvg_loss}
}
\caption[CacheNet's Convergence]{CIFAR-10's and FVG's losses both start high but converge closer and closer to zero. }
\label{fig:cn_loss}
\end{figure}

\begin{figure}[t]
\centering
\subfloat[CIFAR-10]{
\includegraphics[width=0.92\linewidth]{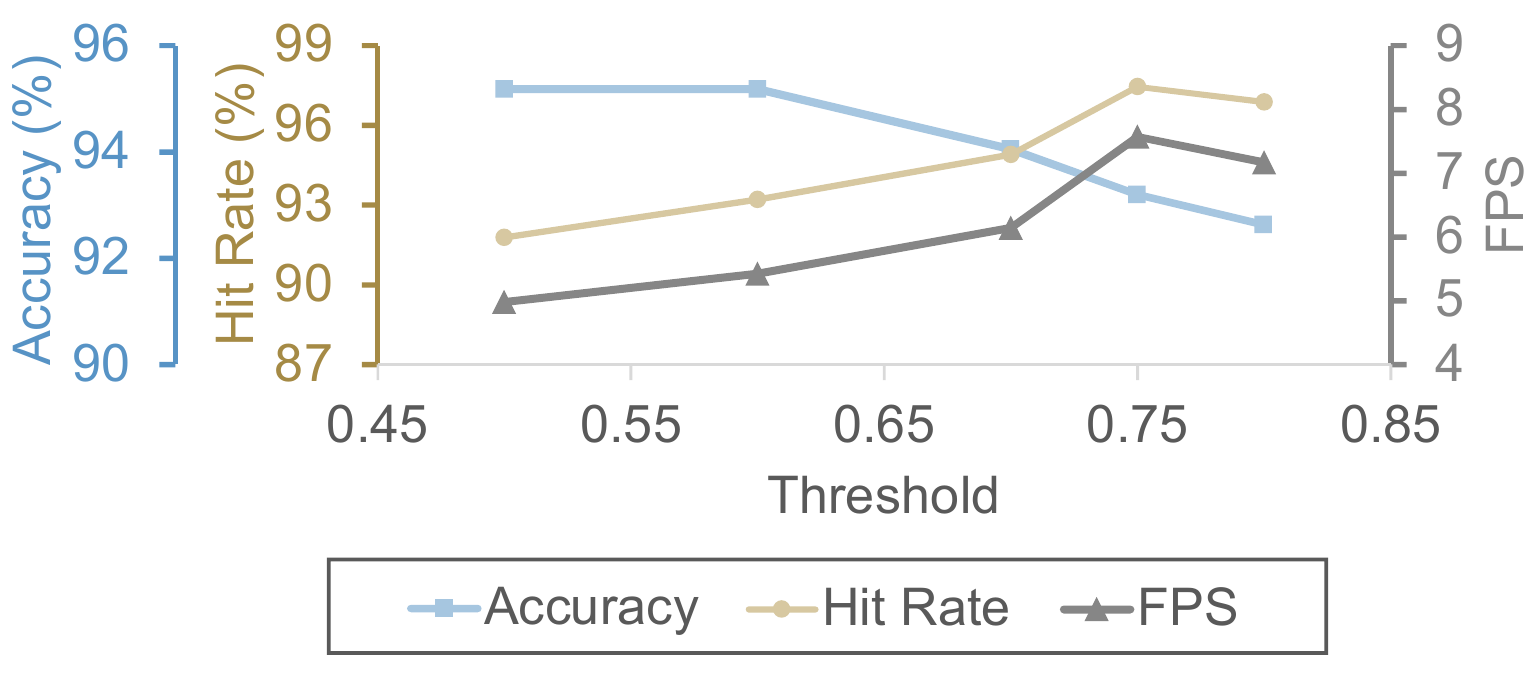}
\label{fig:cn_cifar_fps_hr_acc}
}\\
\subfloat[FVG]{
\includegraphics[width=0.92\linewidth]{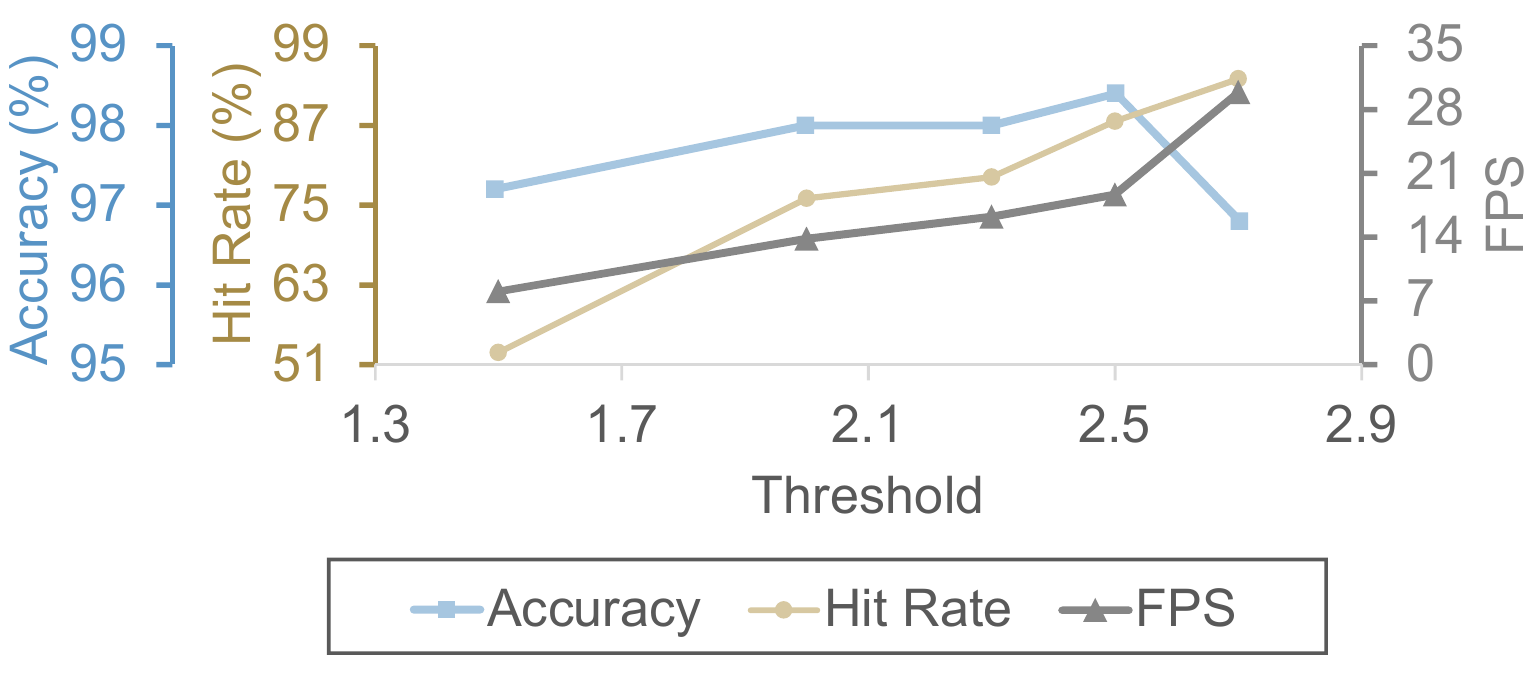}
\label{fig:cn_fvg_fps_hr_acc}
}
\caption[]{FPS and hit rate increase most of the time as the preconfigured threshold increases. Accuracy generally decreases because predictions of less certainty are considered valid. When multiple submodels outperforming the full model (in the FVG dataset), there is a small peak observed before the accuracy declines. }\label{fig:cn_fps_n_acc_n_hit}
\end{figure}

\begin{figure*}[ht]
\ffigbox[]{%
\begin{subfloatrow}
  \ffigbox[0.61\linewidth]
    {\caption{CIFAR-10 FPS}\label{fig:cn_cifar_fps}}
    {\includegraphics[width=1\linewidth]{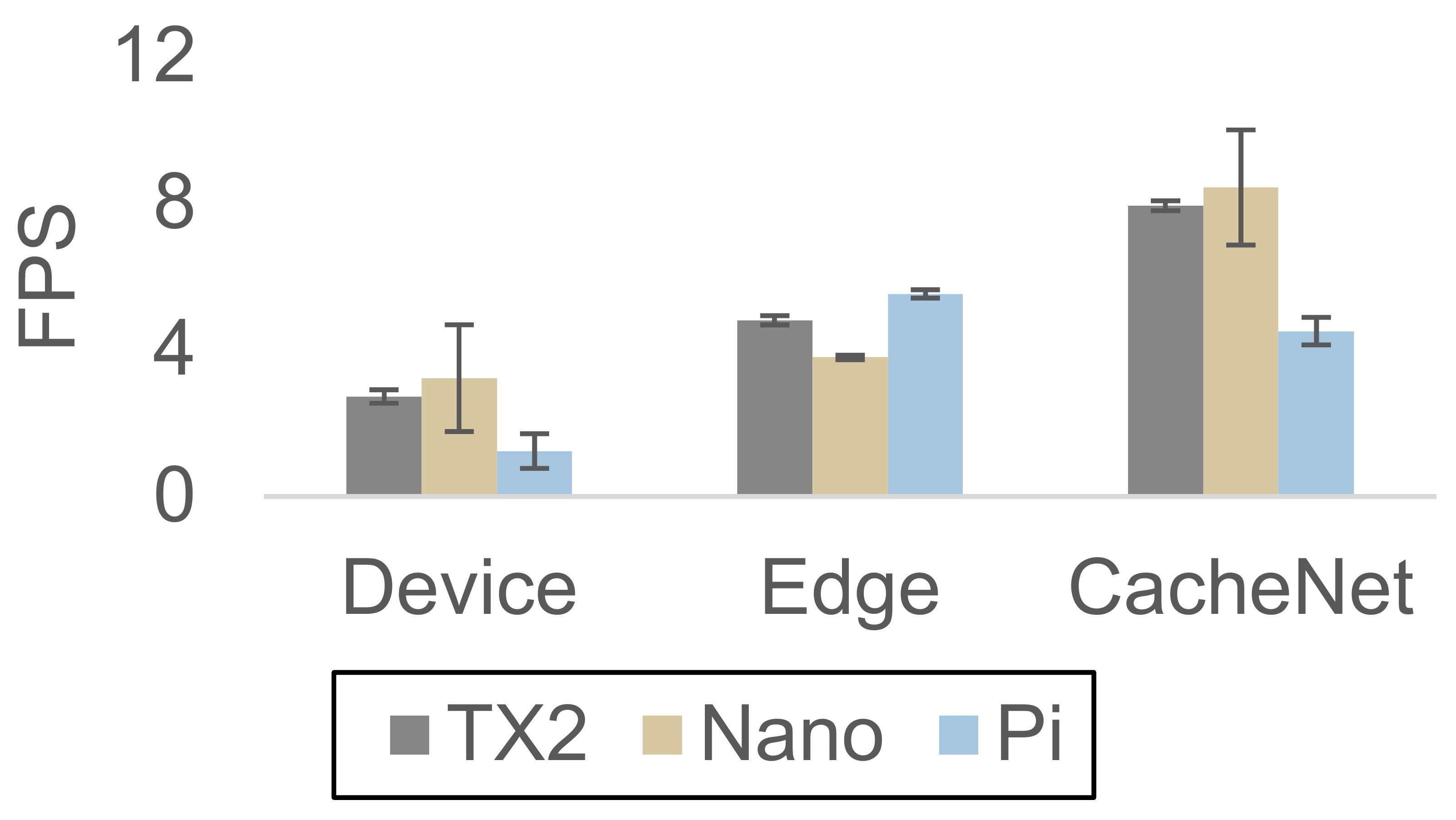}}
  \ffigbox[0.61\linewidth]
    {\caption{FVG FPS}\label{fig:cn_fvg_fps}}
    {\includegraphics[width=1\linewidth]{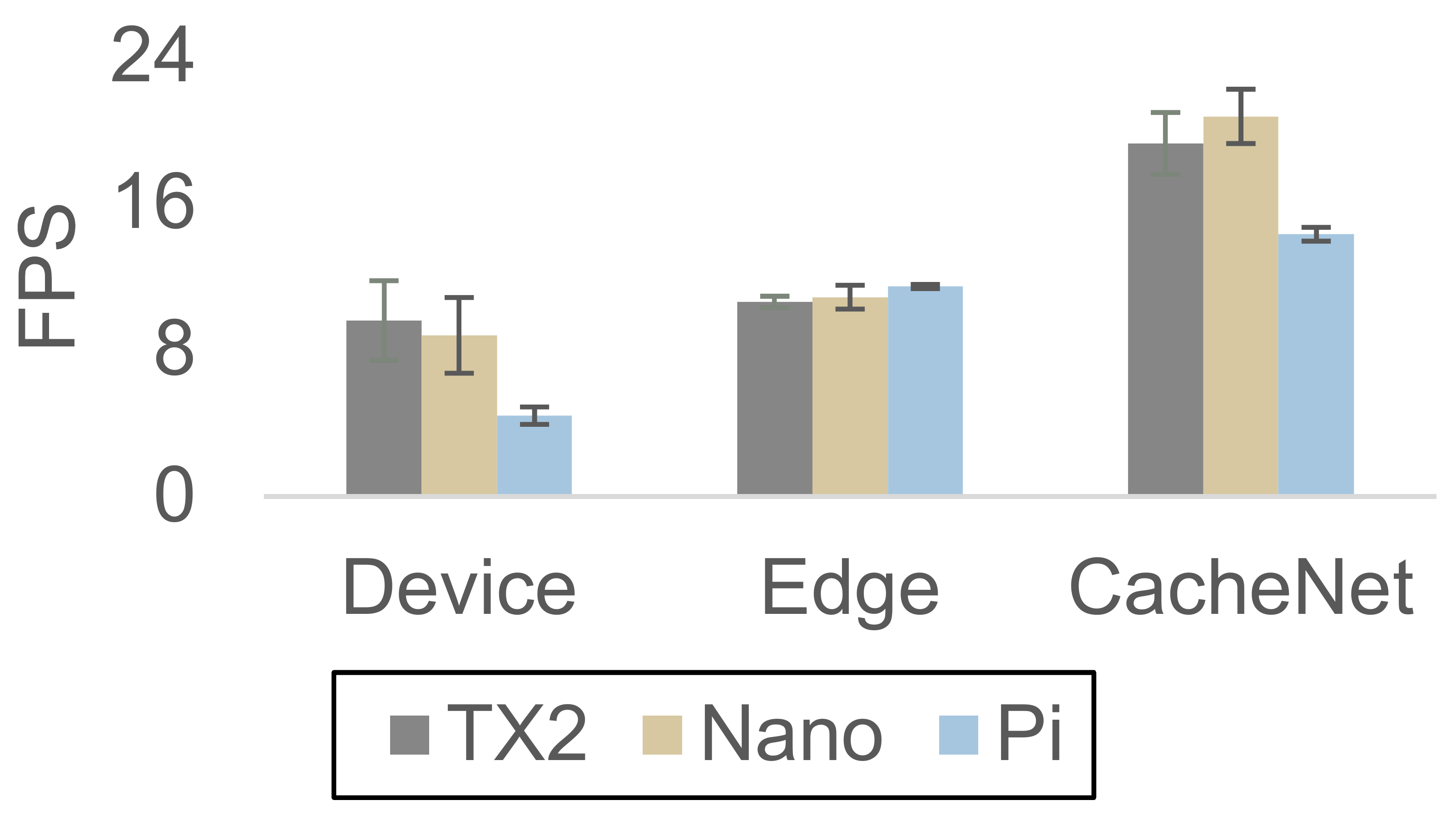}}\hspace{0.023\linewidth}
  \ffigbox[0.61\linewidth]
    {\caption{Accuracy}\label{fig:cn_cifar_fvg_acc}}
    {\includegraphics[width=1\linewidth]{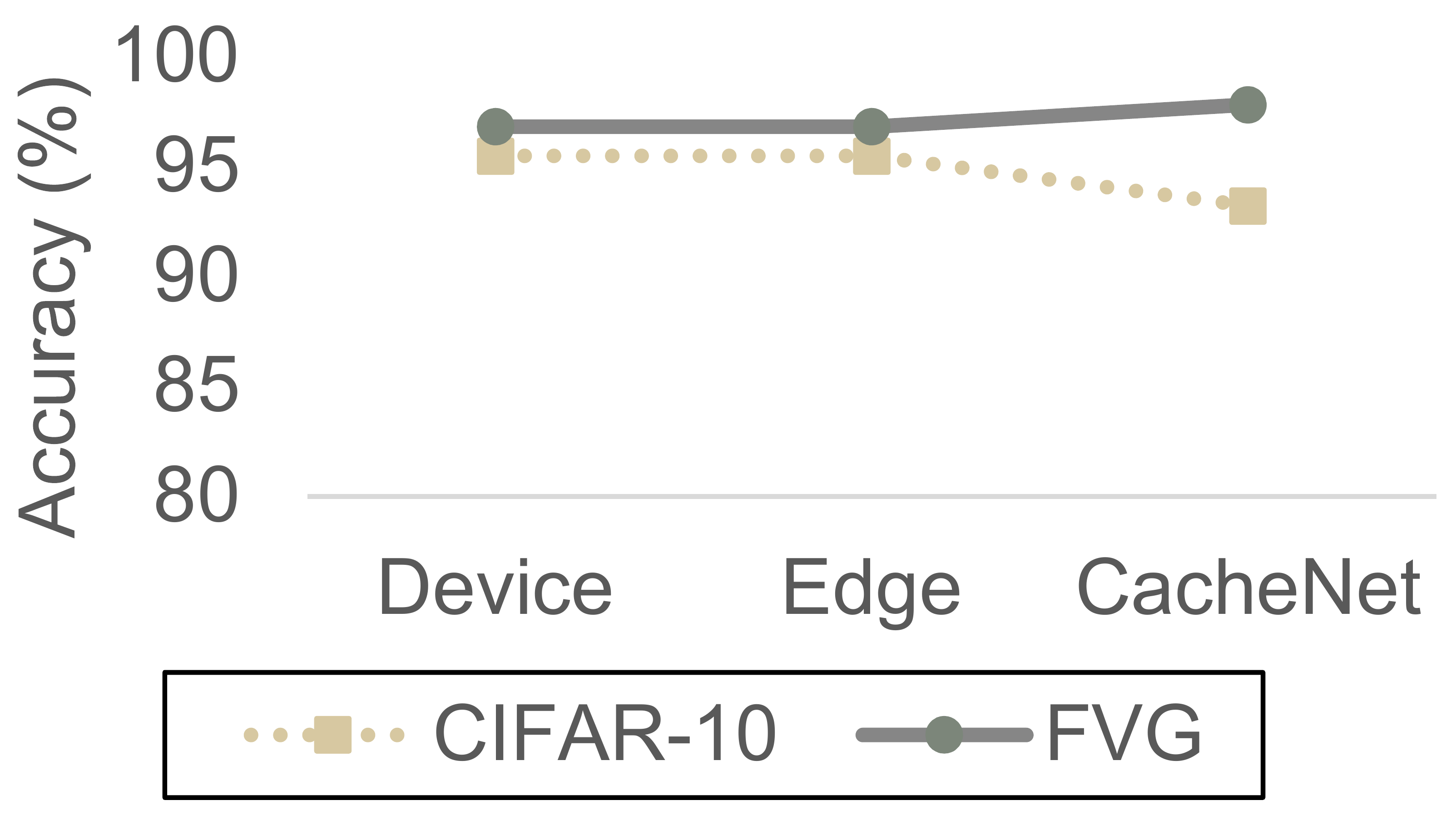}}
\end{subfloatrow}\hspace*{\columnsep}
}{\caption[]{Medians are taken and standard deviations are plotted as error bars. CacheNet is faster than the other two baselines, while accuracy is comparable to the full model.}\label{fig:cn_fps_n_acc}}
\end{figure*}

\paragraph*{Cache replacement} 
As it is discussed in Section~\ref{subsect:cn_cache_replacement}, if the predictive entropy is below a preconfigured threshold, the inference is performed locally; otherwise, it is done remotely on the edge server. Figure~\ref{fig:cn_cifar_fps_hr_acc} and \ref{fig:cn_fvg_fps_hr_acc} demonstrate that the FPS increases as the threshold increased most of the time for both CIFAR-10 and FVG. The reason is that the hit rate is generally higher when the threshold is higher. Fewer cache replacement is needed and more and more images are being processed locally, which speeds up the inference. On the other hand, Figure~\ref{fig:cn_cifar_fps_hr_acc} and \ref{fig:cn_fvg_fps_hr_acc} show that a higher hit rate generally comes at the cost of lower accuracy. It is because a higher threshold allows prediction with higher entropy (uncertainty) to become valid. Higher entropy predictions are of lower quality that decrease the overall accuracy. We find that in practice, it is a trade-off between hit rate and accuracy.

\paragraph*{Comparison to baselines}

A comparison between CacheNet and the other two baselines ({\it Device} and {\it Edge}) are shown in Figure~\ref{fig:cn_cifar_fps} and \ref{fig:cn_fvg_fps}. Medians (of all the scenarios) are taken and standard deviations are plotted as error bars in those figures. For CacheNet, preconfigured threshold $0.75$ and $2.5$ are chosen respectively per CIFAR-10 and FVG to the best extend to trade off hit rate against accuracy. As visualized on those figures, CacheNet is much faster than the other two baselines: for CIFAR-10, $3.2X$ of {\it Device} and $1.6X$ of {\it Edge}; for FVG, $2.5X$ of {\it Device} and $1.7X$ of {\it Edge}. At the same time, the accuracy of CacheNet is comparable with that of the full model, with only a slight drop on CIFAR-10, but increasing a bit on FVG.

More details are given in Table~\ref{tab:cn_cifar_ncnn}--\ref{tab:cn_fvg_fps30_tflite}. CacheNet generally works better on end devices with more computing power such as Jetson TX2 and Jetson Nano. Offloading to the edge server ({\it Edge}) releases end devices' burden thus CPU usages are lowest among three. However, it also implies that the computing power on the end device has not been fully utilized. Memory usages fall into a similar pattern as that of CPU usages. If we divide elapsed time into the time that is run on the end device and that is performed on the edge server (including time for upload and download), we observe that CacheNet distributes the total (computation) time between the end device and the edge server, while the other two baselines are not taking the advantages of distributed computing, that either runs locally ({\it Device}) or computes on the edge server most of the time ({\it Edge}).

\paragraph*{Comparison across frameworks} NCNN and TensorFlow Lite are both lightweight deep learning framework tailored for embedded devices with limited compute power, memory and storage. A comparison between TensorFlow Lite and NCNN are given in Table~\ref{tab:cn_cifar_ncnn}--\ref{tab:cn_fvg_fps30_tflite}. CacheNet with NCNN and TensorFlow Lite both outperform the baselines. NCNN is slightly more efficient than TensorFlow Lite for both CIFAR-10 and FVG, while TensorFlow Lite consumes far less memory than NCNN. 

\paragraph*{Comparison across devices} From Figure~\ref{fig:cn_fps_n_acc}, we observe that CacheNet performs better on end devices with higher compute power such as Jetson TX2 and Jetson Nano. Raspberry Pi incurs more time on submodel inference, which leads to lower FPS. Detailed numerical comparisons can be found in Table~\ref{tab:cn_cifar_ncnn}--\ref{tab:cn_fvg_fps30_tflite}.

\begin{table*}[ht]
\centering
\setlength\tabcolsep{4pt}

\caption[Experimental Results with CIFAR-10 on Jetson TX2, Jetson Nano, and Raspberry Pi 4 - NCNN]{Experimental Results with CIFAR-10 on Jetson TX2, Jetson Nano, and Raspberry Pi 4 - NCNN}

\begin{tabular}{|l|rrr|rrr|rrr|}
\hline
&\multicolumn{3}{|c|}{\bf Jetson TX2} &\multicolumn{3}{|c|}{\bf Jetson Nano} &\multicolumn{3}{|c|}{\bf Raspberry Pi 4} \\
\hline
& \bf Device & \bf Edge & \bf CacheNet & \bf Device & \bf Edge & \bf CacheNet & \bf Device & \bf Edge & \bf CacheNet   \\ 
\hline
\bf FPS&2.85&4.89&\bf8.02&4.25&3.83&\bf9.53&1.57&\bf5.60&4.77\\
\bf Accuracy (\%)&95.47&95.47&93.20&95.47&95.47&93.20&95.47&95.47&93.20\\
\bf CPU (\%)&84.53&4.25&60.26&96.83&5.96&57.23&98.65&1.21&63.75\\
\bf Memory (Mb)&610.71&1.86&198.76&863.14&1.94&241.80&875.53&0.91&201.75\\
\bf Time (s)&124.06&72.16&\bf44.03&83.06&92.20&\bf37.04&224.15&\bf63.02&74.03\\
\bf \branch{} Device (s)&124.06&0.80&26.25&83.06&0.66&17.89&224.15&0.63&42.28\\
\bf \leaf{} Edge (s)&0.00&71.37&17.79&0.00&91.54&19.14&0.00&62.39&31.75\\
\hline
\end{tabular}

\label{tab:cn_cifar_ncnn}
\end{table*}

\begin{table*}[ht]
\centering
\setlength\tabcolsep{4pt}

\caption[Experimental Results with CIFAR-10 on Jetson TX2, Jetson Nano, and Raspberry Pi 4 - TensorFlow Lite]{Experimental Results with CIFAR-10 on Jetson TX2, Jetson Nano, and Raspberry Pi 4 - TensorFlow Lite}

\begin{tabular}{|l|rrr|rrr|rrr|}
\hline
&\multicolumn{3}{|c|}{\bf Jetson TX2} &\multicolumn{3}{|c|}{\bf Jetson Nano} &\multicolumn{3}{|c|}{\bf Raspberry Pi 4} \\
\hline
& \bf Device & \bf Edge & \bf CacheNet & \bf Device & \bf Edge & \bf CacheNet & \bf Device & \bf Edge & \bf CacheNet   \\ 
\hline
\bf FPS&2.59&4.71&\bf7.83&2.19&3.74&\bf7.31&0.90&\bf5.44&4.24\\
\bf Accuracy (\%)&95.47&95.47&93.20&95.47&95.47&93.20&95.47&95.47&93.20\\
\bf CPU (\%)&77.26&4.40&52.37&79.92&5.90&56.82&74.29&1.66&53.45\\
\bf Memory (Mb)&213.42&29.93&113.96&226.37&108.73&133.23&210.12&106.98&99.98\\
\bf Time (s)&136.05&74.91&\bf45.08&161.04&94.29&\bf48.30&390.31&\bf64.90&83.25\\
\bf \branch{} Device (s)&136.05&0.56&29.00&161.04&0.83&34.16&390.31&0.55&60.36\\
\bf \leaf{} Edge (s)&0.00&74.35&16.07&0.00&93.46&14.13&0.00&64.35&22.89\\
\hline
\end{tabular}

\label{tab:cn_cifar_tflite}
\end{table*}


\begin{table*}[ht]
\centering
\setlength\tabcolsep{4pt}

\caption[Experimental Results with FVG (15 FPS) on Jetson TX2, Jetson Nano, and Raspberry Pi 4 - NCNN]{Experimental Results with FVG (15 FPS) on Jetson TX2, Jetson Nano, and Raspberry Pi 4 - NCNN}

\begin{tabular}{|l|rrr|rrr|rrr|}
\hline
&\multicolumn{3}{|c|}{\bf Jetson TX2} &\multicolumn{3}{|c|}{\bf Jetson Nano} &\multicolumn{3}{|c|}{\bf Raspberry Pi 4} \\
\hline
& \bf Device & \bf Edge & \bf CacheNet & \bf Device & \bf Edge & \bf CacheNet & \bf Device & \bf Edge & \bf CacheNet   \\ 
\hline
\bf FPS&11.36&10.40&\bf20.80&10.41&10.40&\bf22.70&5.10&11.36&\bf14.70\\
\bf Accuracy (\%)&97.20&97.20&98.40&97.20&97.20&98.40&97.20&97.20&98.40\\
\bf CPU (\%)&96.05&8.27&22.35&95.22&11.54&23.37&96.32&4.23&24.37\\
\bf Memory (Mb)&312.35&7.19&10.55&436.91&7.75&11.57&454.79&6.72&8.04\\
\bf Time (s)&22.01&24.05&\bf12.02&24.02&24.04&\bf11.01&49.04&22.01&\bf17.01\\
\bf \branch{} Device (s)&22.01&0.72&2.31&24.02&0.87&1.70&49.04&0.48&4.62\\
\bf \leaf{} Edge (s)&0.00&23.32&9.71&0.00&23.17&9.31&0.00&21.53&12.38\\
\hline
\end{tabular}

\label{tab:cn_fvg_fps15_ncnn}
\end{table*}

\begin{table*}[ht]
\centering
\setlength\tabcolsep{4pt}

\caption[Experimental Results with FVG (15 FPS) on Jetson TX2, Jetson Nano, and Raspberry Pi 4 - TensorFlow Lite]{Experimental Results with FVG (15 FPS) on Jetson TX2, Jetson Nano, and Raspberry Pi 4 - TensorFlow Lite}

\begin{tabular}{|l|rrr|rrr|rrr|}
\hline
&\multicolumn{3}{|c|}{\bf Jetson TX2} &\multicolumn{3}{|c|}{\bf Jetson Nano} &\multicolumn{3}{|c|}{\bf Raspberry Pi 4} \\
\hline
& \bf Device & \bf Edge & \bf CacheNet & \bf Device & \bf Edge & \bf CacheNet & \bf Device & \bf Edge & \bf CacheNet   \\ 
\hline
\bf FPS&7.65&10.72&\bf20.65&7.01&10.65&\bf21.47&3.75&10.71&\bf16.02\\
\bf Accuracy (\%)&97.20&97.20&98.40&97.20&97.20&98.40&97.20&97.20&98.40\\
\bf CPU (\%)&69.51&8.63&18.31&72.55&10.77&21.11&62.75&4.89&18.32\\
\bf Memory (Mb)&194.68&10.61&16.16&181.79&99.08&16.95&190.74&96.64&10.42\\
\bf Time (s)&32.70&23.33&\bf12.11&35.67&23.48&\bf11.64&66.65&23.33&\bf15.60\\
\bf \branch{} Device (s)&32.70&0.49&2.42&35.67&0.74&2.87&66.65&0.80&4.59\\
\bf \leaf{} Edge (s)&0.00&22.84&9.68&0.00&22.74&8.78&0.00&22.53&11.01\\
\hline
\end{tabular}

\label{tab:cn_fvg_fps15_tflite}
\end{table*}


\begin{table*}[ht]
\centering
\setlength\tabcolsep{4pt}

\caption[Experimental Results with FVG (30 FPS) on Jetson TX2, Jetson Nano, and Raspberry Pi 4 - NCNN]{Experimental Results with FVG (30 FPS) on Jetson TX2, Jetson Nano, and Raspberry Pi 4 - NCNN}

\begin{tabular}{|l|rrr|rrr|rrr|}
\hline
&\multicolumn{3}{|c|}{\bf Jetson TX2} &\multicolumn{3}{|c|}{\bf Jetson Nano} &\multicolumn{3}{|c|}{\bf Raspberry Pi 4} \\
\hline
& \bf Device & \bf Edge & \bf CacheNet & \bf Device & \bf Edge & \bf CacheNet & \bf Device & \bf Edge & \bf CacheNet   \\ 
\hline
\bf FPS&11.62&11.09&\bf17.84&10.86&11.88&\bf19.98&5.05&11.62&\bf13.89\\
\bf Accuracy (\%)&96.40&96.40&97.20&96.40&96.40&97.20&96.40&96.40&97.20\\
\bf CPU (\%)&96.94&8.46&22.78&96.91&11.84&24.05&98.10&4.43&25.30\\
\bf Memory (Mb)&310.16&12.76&9.00&455.06&12.79&11.52&454.23&11.48&9.35\\
\bf Time (s)&43.02&45.08&\bf28.03&46.03&42.07&\bf25.03&99.04&43.01&\bf36.01\\
\bf \branch{} Device (s)&43.02&0.87&3.77&46.03&0.65&3.21&99.04&0.22&8.52\\
\bf \leaf{} Edge (s)&0.00&44.21&24.26&0.00&41.42&21.82&0.00&42.79&27.49\\
\hline
\end{tabular}

\label{tab:cn_fvg_fps30_ncnn}
\end{table*}

\begin{table*}[ht]
\centering
\setlength\tabcolsep{4pt}

\caption[Experimental Results with FVG (30 FPS) on Jetson TX2, Jetson Nano, and Raspberry Pi 4 - TensorFlow Lite]{Experimental Results with FVG (30 FPS) on Jetson TX2, Jetson Nano, and Raspberry Pi 4 - TensorFlow Lite}

\begin{tabular}{|l|rrr|rrr|rrr|}
\hline
&\multicolumn{3}{|c|}{\bf Jetson TX2} &\multicolumn{3}{|c|}{\bf Jetson Nano} &\multicolumn{3}{|c|}{\bf Raspberry Pi 4} \\
\hline
& \bf Device & \bf Edge & \bf CacheNet & \bf Device & \bf Edge & \bf CacheNet & \bf Device & \bf Edge & \bf CacheNet   \\ 
\hline
\bf FPS&7.82&10.49&\bf17.76&7.15&11.08&\bf19.43&3.43&11.53&\bf13.49\\
\bf Accuracy (\%)&96.40&96.40&97.20&96.40&96.40&97.20&96.40&96.40&97.20\\
\bf CPU (\%)&69.98&8.60&19.89&73.13&11.52&21.24&63.68&4.52&17.17\\
\bf Memory (Mb)&197.03&10.34&16.64&189.10&99.05&16.09&191.69&6.79&11.23\\
\bf Time (s)&63.90&47.67&\bf28.15&69.89&45.13&\bf25.73&145.61&43.35&\bf37.06\\
\bf \branch{} Device (s)&63.90&0.69&4.97&69.89&0.99&5.57&145.61&0.19&9.29\\
\bf \leaf{} Edge (s)&0.00&46.98&23.18&0.00&44.14&20.16&0.00&43.16&27.78\\
\hline
\end{tabular}

\label{tab:cn_fvg_fps30_tflite}
\end{table*}


\section{Conclusion}
\label{sect:cn_conclusion}
In this paper, we proposed CacheNet, a neural network model caching mechanism for edge computing. In CacheNet, an edge (cloud) server is responsible for the storage and selection of neural network partitions, while an end device with a cached partition performs inferencing most of the time. 

Three key features enable CacheNet to achieve short end-to-end latency without much compromise in prediction accuracy: 1) Caching avoids the communication latency between an end device and edge (cloud) server whenever there is a cache hit; 2) specialized cached partitions do not sacrifice prediction accuracy if properly trained and selected; 3) the computation and storage complexities of cached model partitions are smaller rather than those of a full model. 

In future works, we plan to experiment with more datasets and neural network models using CacheNet. The two-level caching idea can be further extended to consider a hierarchy of caches, e.g., distributed among end devices, edge nodes and cloud servers. Another line of research is to apply neural architecture search to CacheNet to improve its adaptability to different types of neural networks.


%


\ifCLASSOPTIONcompsoc
  \section*{Acknowledgments}
\else
  \section*{Acknowledgment}
\fi

This work is in part supported by the Discovery Grant and Collaborative Research Development Grant from Natural Science and Engineering Council, Canada. 

The authors would like to thank McMaster Faculty of Engineering SummerTech Entrepreneur Fellowship for offering the financial support in purchasing experimental equipment including a Jetson Nano, Raspberry Pi 4, and a TP-LINK Archer C3200 router.

\appendices
\section{Absence of B\'{e}l\'{a}dy's Anomaly}
\label{sect:appendix}
B\'{e}l\'{a}dy's anomaly is the phenomenon that a larger cache incurs more cache misses than a smaller one. In CacheNet, there are two possible ways to take advantage of a larger cache size: 1) each individual submodel being cached has a larger capacity (i.e., deeper); 2) more submodels are being cached on an end device. If both do not result in fewer cache hits, we can conclude that B\'{e}l\'{a}dy's anomaly does not occur in CacheNet. 

\subsection{Larger Capacity}
A submodel with a larger capacity is defined as follows.
Given any sequence $X = x_1, x_2, \ldots , x_N$ of images, audio clips etc. Let $\Phi = \phi^{(1)}, \phi^{(2)}, \ldots, \phi^{(Q)}$ be an sequence of submodel instances for caching, with respect to 1) their depths $d^{(1)} < d^{(2)} < \ldots < d^{(Q)}$, 2) any layer in $\phi^{(1)}$ contained by $\phi^{(2)}$, $\ldots$, and any layer in $\phi^{(Q-1)}$ contained by $\phi^{(Q)}$. According to the capacity theorem \cite{cohen2016expressive}, submodel instance $\phi^{(1)}$ expresses less functions than $\phi^{(2)}$, $\ldots$, and $\phi^{(Q-1)}$ less functions than $\phi^{(Q)}$.

Let $H(\hat{y}|x_i, \phi^{(j)})$ be the predictive entropy given any input $x_i, i=1,2,\ldots,N$ and any submodel instance $\phi^{(j)}, j=1,2,\ldots,Q$. For a predefined threshold $T$, if $H(\hat{y}|x_i, \phi^{(j)}) < T$, we say it is a cache hit, else it is a cache miss. 

\begin{theorem}
Let $M(X, \phi^{(j)})$ be the number of misses (faults) given the input sequence $X$ and the submodel instance $\phi^{(j)}, j=1,2,\ldots,Q$. Then 
$M(X, \phi^{(1)}) \geq M(X, \phi^{(2)}) \geq \ldots \geq M(X, \phi^{(Q)})$

\end{theorem}

\begin{proof}
We can prove this theorem by induction. 

1) Base case: if $X = x_1$, both $\phi^{(j)}$ and $\phi^{(j+1)}$ incurs a cache miss on $x_1$, thus, $M(X, \phi^{(j)}) = M(X, \phi^{(j+1)})$

2) Induction hypothesis: we need to show if $X = x_1, \ldots, x_i$, $M(X, \phi^{(j)}) \geq M(X, \phi^{(j+1)})$ for an arbitrary $j$, when $X = x_1, \ldots, x_{i+1}$, $M(X, \phi^{(j)}) \geq M(X, \phi^{(j+1)})$ also holds.

a) If the newly input $x_{i+1}$ incurs a cache hit on the submodel instance $\phi^{(j)}$, there should be also a cache hit on $\phi^{(j+1)}$. This claim relies on the capacity theorem \cite{cohen2016expressive} that the submodel instance $\phi^{(j+1)}$ has more functional expressibility than $\phi^{(j)}$. By definition, the submodel instance $\phi^{(j)}$ can be embedded in $\phi^{(j+1)}$. The submodel instance $\phi^{(j+1)}$'s additional layers can be made as an identity for $x_1, \ldots, x_{i+1}$'s intermediate outputs. Thus, the claim holds. 

b) If the new input $x_{i+1}$ incurs a cache miss on $\phi^{(j)}$, there may be a cache hit or cache miss on $\phi^{(j+1)}$. Since the submodel instance $\phi^{(j)}$ is embedded in $\phi^{(j+1)}$, and $\phi^{(j+1)}$'s additional layers are made as an identity for $x_1, \ldots, x_i$'s intermediate outputs. The additional layers of $\phi^{(j+1)}$ may have the additional capacity to represent $x_{i+1}$'s function. 

In either case, $M(X, \phi^{(j)}) \geq M(X, \phi^{(j+1)})$ for an arbitrary $j$. The induction hypothesis holds. 
\end{proof}


\subsection{More Submodels}

When there are multiple submodels to cache on an end device, a cache miss happens if the predictive entropy of the current submodel is less than the threshold $T$ and there is no suitable submodel (which is decided by the S-InfoVAE on the end device) currently stored on the end device. 

\begin{theorem}
Let $k$ ($1 \leq k \leq K$) be the number of submodels cached on an end device. 
Let $\bar{M}(X, k)$ be the number of misses (faults) given the input sequence $X$. Then, under the LRU cache replacement policy,  
$\bar{M}(X, 1) \geq \bar{M}(X, 2) \geq \dots \geq \bar{M}(X, K)$.
\end{theorem}

\begin{proof}
We can prove this theorem by induction. 

1) Base case: if $X = x_1$, both $k$ and $k+1$ cached submodels incur a cache miss on $x_1$, thus, $\bar{M}(X, k) = \bar{M}(X, k+1)$

2) Induction hypothesis: we need to show if $X = x_1, \ldots, x_i$, $\bar{M}(X, k) \geq \bar{M}(X, k+1)$ for an arbitrary $k$, when $X = x_1, \ldots, x_{i+1}$, $\bar{M}(X, k) \geq \bar{M}(X, k+1)$ also holds.

a) If the newly input $x_{i+1}$ incurs a cache hit on $k$ cached submodels, there should be also a cache hit on $k+1$ cached submodels, because the $k$ cached submodels are always embedded in the $k+1$ submodels under the least recently used (LRU) policy. 

b) If the newly input $x_{i+1}$ incurs a cache miss on $k$ cached submodels, there may be a cache hit or cache miss on $k+1$ cached submodels, because the $k$ submodels are embedded in the $k+1$ submodels, the one more submodel in the cached $k+1$ submodels may cause the hit or not depending on whether it matches the index given by S-InfoVAE.  

No matter in either case, $\bar{M}(X, k) \geq \bar{M}(X, k+1)$ for an arbitrary $k$. The induction hypothesis holds. Thus, the theorem holds. 
\end{proof}

To this end, we conclude when individual submodels have larger capacity or more submodels can be cached on an end device, CacheNet always has higher or the same hit rates. In other words, it does not suffer from  B\'{e}l\'{a}dy's anomaly.

\ifCLASSOPTIONcaptionsoff
  \newpage
\fi



%

\bibliographystyle{plain}
\bibliography{references}

\begin{thebibliography}{10}

\bibitem{apicharttrisorn2019frugal}
Kittipat Apicharttrisorn, Xukan Ran, Jiasi Chen, Srikanth~V Krishnamurthy, and
  Amit~K Roy-Chowdhury.
\newblock Frugal following: Power thrifty object detection and tracking for
  mobile augmented reality.
\newblock In {\em Proceedings of the 17th Conference on Embedded Networked
  Sensor Systems}, pages 96--109, 2019.

\bibitem{bao2019depth}
Wenbo Bao, Wei-Sheng Lai, Chao Ma, Xiaoyun Zhang, Zhiyong Gao, and Ming-Hsuan
  Yang.
\newblock Depth-aware video frame interpolation.
\newblock In {\em Proceedings of the IEEE Conference on Computer Vision and
  Pattern Recognition}, pages 3703--3712, 2019.

\bibitem{chen2015glimpse}
Tiffany Yu-Han Chen, Lenin Ravindranath, Shuo Deng, Paramvir Bahl, and Hari
  Balakrishnan.
\newblock Glimpse: Continuous, real-time object recognition on mobile devices.
\newblock In {\em Proceedings of the 13th ACM Conference on Embedded Networked
  Sensor Systems}, pages 155--168. ACM, 2015.

\bibitem{cohen2016expressive}
Nadav Cohen, Or~Sharir, and Amnon Shashua.
\newblock On the expressive power of deep learning: A tensor analysis.
\newblock In {\em Conference on Learning Theory}, pages 698--728, 2016.

\bibitem{kingma2013auto}
P~Kingma Diederik, Max Welling, et~al.
\newblock Auto-encoding variational bayes.
\newblock In {\em Proceedings of the International Conference on Learning
  Representations (ICLR)}, 2014.

\bibitem{fang2019teamnet}
Yihao Fang, Ziyi Jin, and Rong Zheng.
\newblock Teamnet: A collaborative inference framework on the edge.
\newblock In {\em 2019 IEEE 39th International Conference on Distributed
  Computing Systems (ICDCS)}, pages 1487--1496. IEEE, 2019.

\bibitem{gastaldi2017shake}
Xavier Gastaldi.
\newblock Shake-shake regularization.
\newblock {\em arXiv preprint arXiv:1705.07485}, 2017.

\bibitem{guo2018foggycache}
Peizhen Guo, Bo~Hu, Rui Li, and Wenjun Hu.
\newblock Foggycache: Cross-device approximate computation reuse.
\newblock In {\em Proceedings of the 24th Annual International Conference on
  Mobile Computing and Networking}, pages 19--34, 2018.

\bibitem{he2016deep}
Kaiming He, Xiangyu Zhang, Shaoqing Ren, and Jian Sun.
\newblock Deep residual learning for image recognition.
\newblock In {\em Proceedings of the IEEE conference on computer vision and
  pattern recognition}, pages 770--778, 2016.

\bibitem{huynh2017deepmon}
Loc~N Huynh, Youngki Lee, and Rajesh~Krishna Balan.
\newblock Deepmon: Mobile gpu-based deep learning framework for continuous
  vision applications.
\newblock In {\em Proceedings of the 15th Annual International Conference on
  Mobile Systems, Applications, and Services}, pages 82--95. ACM, 2017.

\bibitem{ioffe2015batch}
Sergey Ioffe and Christian Szegedy.
\newblock Batch normalization: Accelerating deep network training by reducing
  internal covariate shift.
\newblock {\em arXiv preprint arXiv:1502.03167}, 2015.

\bibitem{kang2017neurosurgeon}
Yiping Kang, Johann Hauswald, Cao Gao, Austin Rovinski, Trevor Mudge, Jason
  Mars, and Lingjia Tang.
\newblock Neurosurgeon: Collaborative intelligence between the cloud and mobile
  edge.
\newblock In {\em Proceedings of the Twenty-Second International Conference on
  Architectural Support for Programming Languages and Operating Systems}, pages
  615--629. ACM, 2017.

\bibitem{krizhevsky2009learning}
Alex Krizhevsky, Geoffrey Hinton, et~al.
\newblock Learning multiple layers of features from tiny images.
\newblock Technical report, Citeseer, 2009.

\bibitem{kwon2008tracking}
Junseok Kwon and Kyoung~Mu Lee.
\newblock Tracking of abrupt motion using wang-landau monte carlo estimation.
\newblock In {\em European conference on computer vision}, pages 387--400.
  Springer, 2008.

\bibitem{lin2016fixed}
Darryl Lin, Sachin Talathi, and Sreekanth Annapureddy.
\newblock Fixed point quantization of deep convolutional networks.
\newblock In {\em International Conference on Machine Learning}, pages
  2849--2858, 2016.

\bibitem{mathur2017deepeye}
Akhil Mathur, Nicholas~D Lane, Sourav Bhattacharya, Aidan Boran, Claudio
  Forlivesi, and Fahim Kawsar.
\newblock Deepeye: Resource efficient local execution of multiple deep vision
  models using wearable commodity hardware.
\newblock In {\em Proceedings of the 15th Annual International Conference on
  Mobile Systems, Applications, and Services}, pages 68--81. ACM, 2017.

\bibitem{mutlu2015research}
Onur Mutlu and Lavanya Subramanian.
\newblock Research problems and opportunities in memory systems.
\newblock {\em Supercomputing frontiers and innovations}, 1(3):19--55, 2015.

\bibitem{sainath2013low}
Tara~N Sainath, Brian Kingsbury, Vikas Sindhwani, Ebru Arisoy, and Bhuvana
  Ramabhadran.
\newblock Low-rank matrix factorization for deep neural network training with
  high-dimensional output targets.
\newblock In {\em 2013 IEEE international conference on acoustics, speech and
  signal processing}, pages 6655--6659. IEEE, 2013.

\bibitem{shorten2019survey}
Connor Shorten and Taghi~M Khoshgoftaar.
\newblock A survey on image data augmentation for deep learning.
\newblock {\em Journal of Big Data}, 6(1):60, 2019.

\bibitem{teerapittayanon2017distributed}
Surat Teerapittayanon, Bradley McDanel, and HT~Kung.
\newblock Distributed deep neural networks over the cloud, the edge and end
  devices.
\newblock In {\em Distributed Computing Systems (ICDCS), 2017 IEEE 37th
  International Conference on}, pages 328--339. IEEE, 2017.

\bibitem{xu2018deepcache}
Mengwei Xu, Mengze Zhu, Yunxin Liu, Felix~Xiaozhu Lin, and Xuanzhe Liu.
\newblock Deepcache: principled cache for mobile deep vision.
\newblock In {\em Proceedings of the 24th Annual International Conference on
  Mobile Computing and Networking}, pages 129--144. ACM, 2018.

\bibitem{gait-recognition-via-disentangled-representation-learning}
Ziyuan Zhang, Luan Tran, Xi~Yin, Yousef Atoum, Jian Wan, Nanxin Wang, and
  Xiaoming Liu.
\newblock Gait recognition via disentangled representation learning.
\newblock In {\em In Proceeding of IEEE Computer Vision and Pattern
  Recognition}, Long Beach, CA, June 2019.

\bibitem{zhao2019infovae}
Shengjia Zhao, Jiaming Song, and Stefano Ermon.
\newblock Infovae: Balancing learning and inference in variational
  autoencoders.
\newblock In {\em Proceedings of the AAAI Conference on Artificial
  Intelligence}, volume~33, pages 5885--5892, 2019.

\end{thebibliography}

%

\begin{IEEEbiography}[{\includegraphics[width=1in,height=1.25in,clip,keepaspectratio]{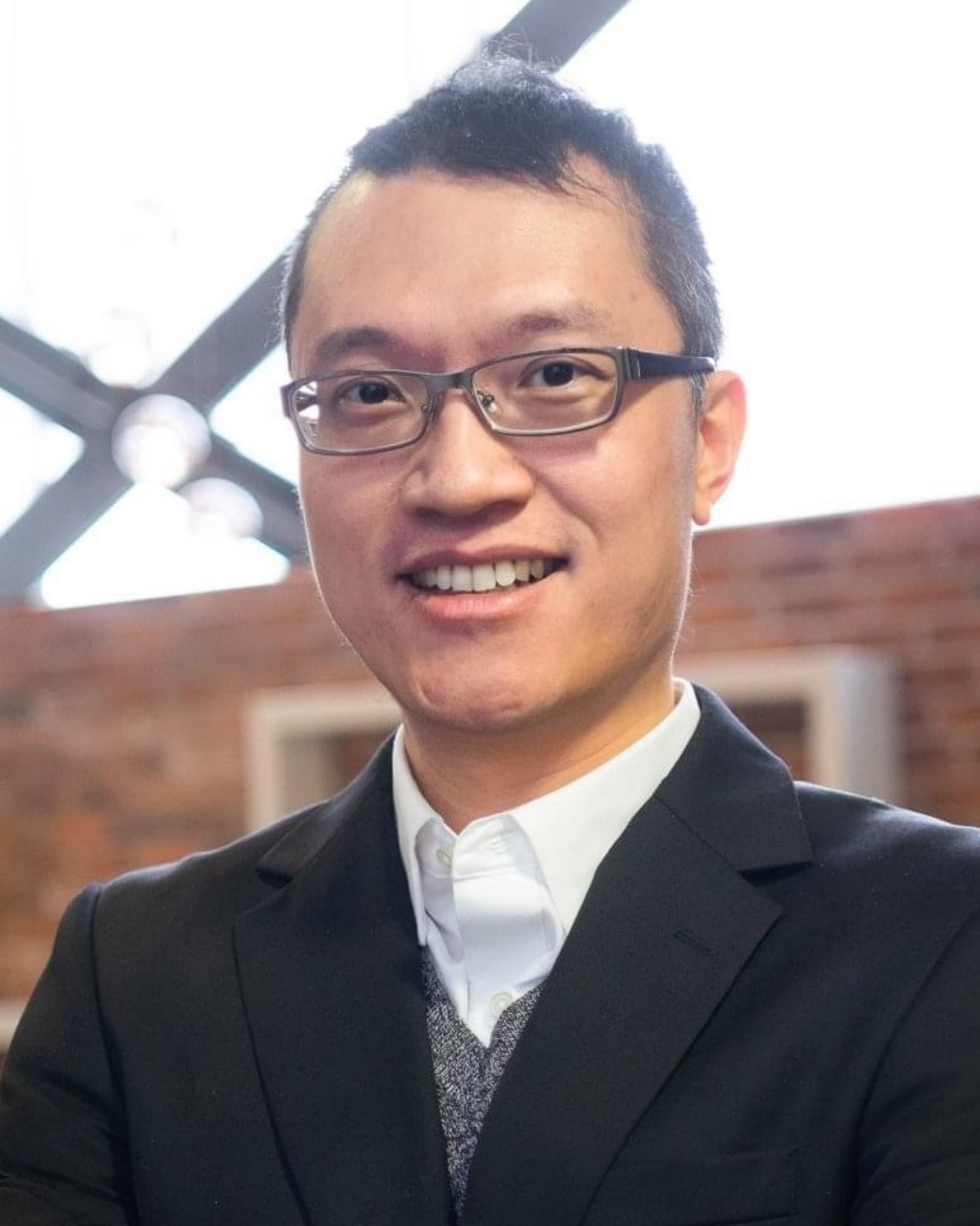}}]{Yihao Fang}
received his M.Eng. from the Dept. of Computing and Software, McMaster University and his B.Eng. from Tongji University, P.R. China. He is currently working toward his Ph.D. in the Dept. of Computing and Software, McMaster University. In the past, Yihao Fang has worked with leading companies in the industry such as Hewlett-Packard and Oracle. He is an IEEE student member and reviewer, and the CEO and founder of an Artificial Intelligence startup company: Lexivalley Inc. 
\end{IEEEbiography}

\begin{IEEEbiography}[{\includegraphics[width=1in,height=1.25in,clip,keepaspectratio]{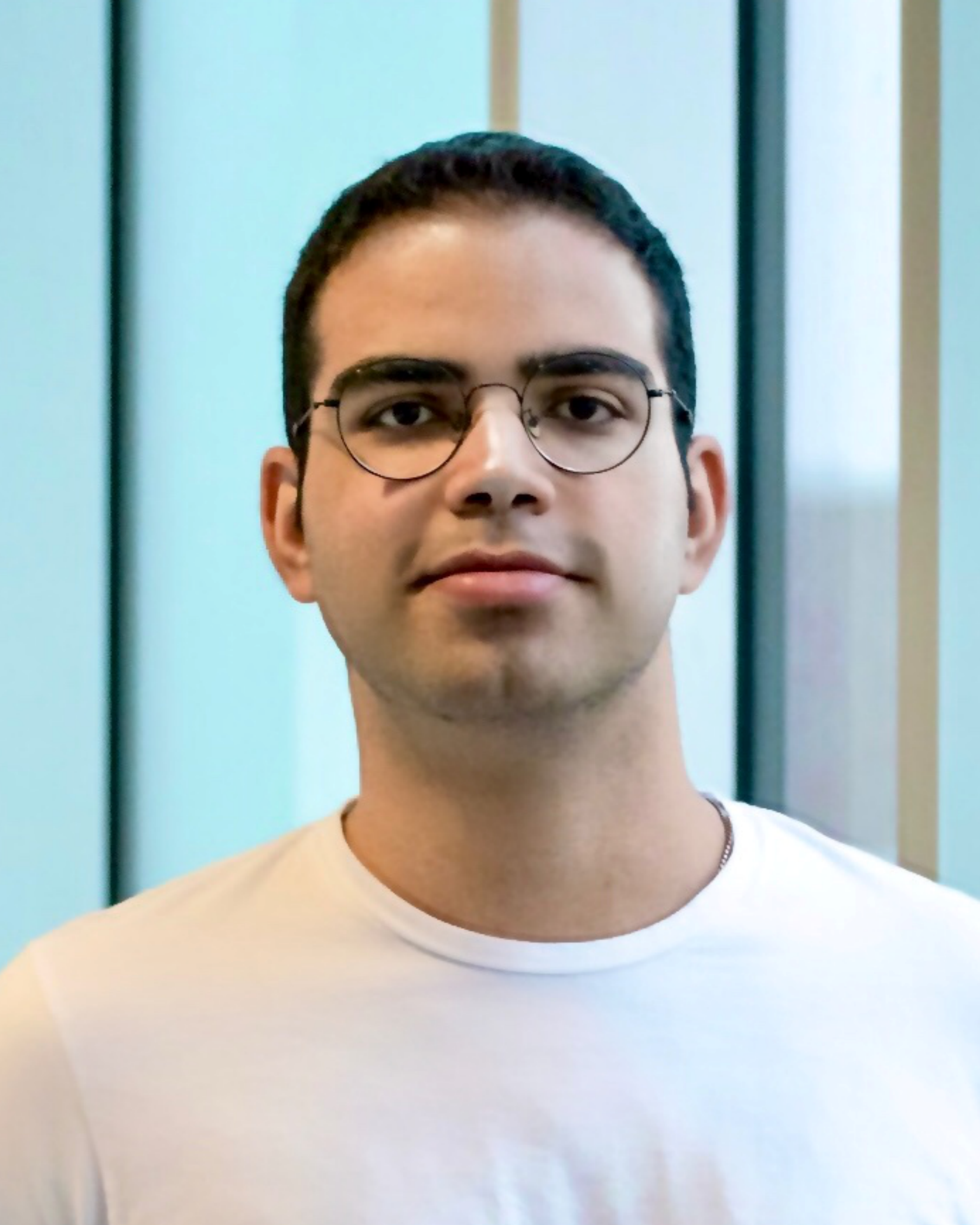}}]{Shervin Manzuri Shalmani}
received his B.Sc. in Computer Engineering from Sharif University of Technology, Tehran, Iran, in 2018. He is now pursuing his M.Sc. in Computer Science at McMaster University. His research areas of interest include deep learning, video processing and computer vision.
\end{IEEEbiography}


\begin{IEEEbiography}[{\includegraphics[width=1in,height=1.25in,clip,keepaspectratio]{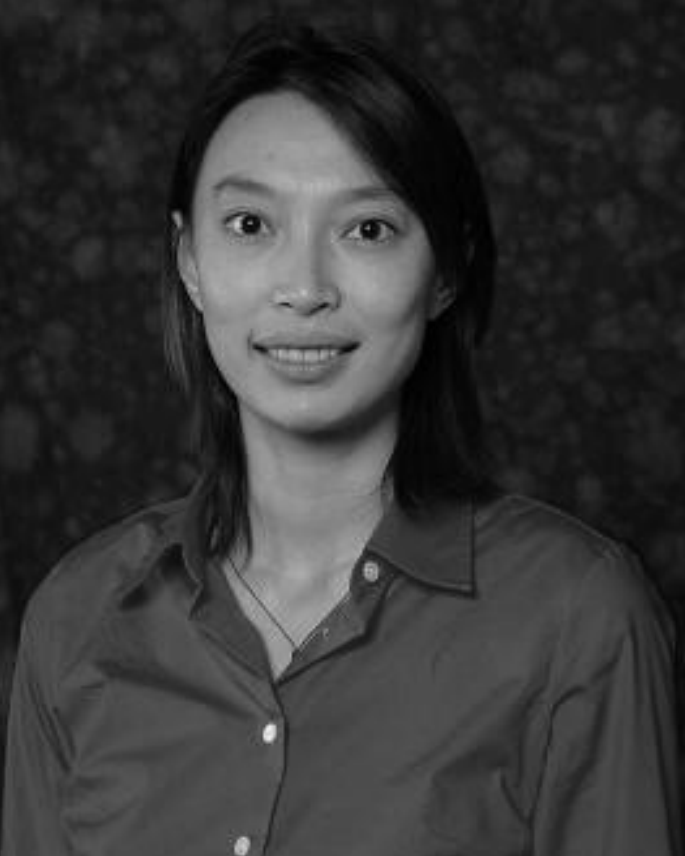}}]{Rong Zheng}
is a Professor in the Dept. of Computing and Software, McMaster University. She is an expert in wireless networking, mobile computing and mobile data analytics. She received the National Science Foundation CAREER Award in 2006, and was a Joseph Ip Distinguished Engineering Fellow from 2015 - 2018. Dr. Zheng is currently an editor of IEEE Transactions on Mobile Computing.
\end{IEEEbiography}


\vfill


\end{document}